\documentclass[nohyperref]{article}
\usepackage{microtype}
\usepackage{graphicx}
\usepackage{subfigure}
\usepackage{booktabs} 
\usepackage{amsmath}
\usepackage{amssymb}
\usepackage{amsthm}
\usepackage{bbm}
\usepackage{enumerate}
\usepackage{paralist}
\usepackage{bbding}
\usepackage{multirow}
\usepackage{arydshln}
\usepackage{pifont}
\usepackage[table]{xcolor}
\usepackage{fontawesome}
\theoremstyle{definition}
\usepackage[english]{babel}
\newtheorem{definition}{Definition}
\newtheorem{theorem}{Theorem}

\newtheorem{lemma}{Lemma}

\usepackage{hyperref}
\usepackage[linesnumbered,ruled,commentsnumbered]{algorithm2e}
\newcommand{\lidy}[1]{\textcolor{blue}{#1}}

\usepackage[accepted]{icml2021}

\icmltitlerunning{Community-Invariant Graph Contrastive Learning}

\begin{document}

\twocolumn[
\icmltitle{Community-Invariant Graph Contrastive Learning}




\icmlsetsymbol{equal}{*}
\icmlsetsymbol{corre}{$\dagger$}

\begin{icmlauthorlist}
\icmlauthor{Shiyin Tan}{equal,to}
\icmlauthor{Dongyuan Li}{equal,to}
\icmlauthor{Renhe Jiang}{corre,goo}
\icmlauthor{Ying Zhang}{ed}
\icmlauthor{Manabu Okumura}{to}
\end{icmlauthorlist}

\icmlaffiliation{to}{Tokyo Institute of Technology}
\icmlaffiliation{goo}{The University of Tokyo}
\icmlaffiliation{ed}{RIKEN \& Tohoku University}

\icmlcorrespondingauthor{Renhe Jiang}{jiangrh@csis.u-tokyo.ac.jp}

\icmlkeywords{Machine Learning, ICML}

\vskip 0.3in
]
\printAffiliationsAndNotice{\icmlEqualContribution}

\begin{abstract}

Graph augmentation has received great attention in recent years for graph contrastive learning (GCL) to learn well-generalized node/graph representations.
However, mainstream GCL methods often favor randomly disrupting graphs for augmentation, which shows limited generalization and inevitably leads to the corruption of high-level graph information, i.e., the graph community.
Moreover, current knowledge-based graph augmentation methods can only focus on either topology or node features, causing the model to lack robustness against various types of noise.
To address these limitations, this research investigated the role of the graph community in graph augmentation and figured out its crucial advantage for learnable graph augmentation. Based on our observations, we propose a community-invariant GCL framework to maintain graph community structure during learnable graph augmentation. By maximizing the spectral changes, this framework unifies the constraints of both topology and feature augmentation, enhancing the model's robustness.
Empirical evidence on 21 benchmark datasets demonstrates the exclusive merits of our framework. 
Code is released on Github [{\href{https://github.com/ShiyinTan/CI-GCL.git}{https://github.com/CI-GCL.git}}].
\end{abstract}

\section{Introduction}
\label{sec:introduction}



Graph representation learning on graph-structured data, such as molecules and social networks, has become one of the hottest topics in AI~\cite{DBLP:conf/icml/CaoCLJ23}. Typical GNNs~\cite{DBLP:conf/iclr/KipfW17} require large-scale task-specific labels, which are expensive and labor-intensive to collect. To alleviate this, graph contrastive learning (GCL) has been proposed as one of the most successful graph representation learning methods, drawing a lot of attention~\cite{DBLP:conf/icml/LiWZW0C22}. The main goal of GCL is to maximize the agreement of node representations between two augmented views to capture graph invariance information~\cite{DBLP:conf/nips/Tian0PKSI20}. Among various GCL variations, effective graph augmentation turns out to be the bread and butter for achieving success~\cite{DBLP:conf/icml/WeiWBNBF23}. Early studies almost adopt random graph augmentation, such as randomly dropping edges or masking features~\cite{DBLP:conf/nips/YouCSCWS20}. Researchers also attempt to incorporate expert knowledge to guide graph augmentation. For instance, GCA~\cite{DBLP:conf/www/0001XYLWW21} and MoCL~\cite{DBLP:conf/kdd/SunXWCZ21} use network science or biomedical knowledge to constrain edge dropping probabilities. 
However, such random or knowledge-based graph augmentations are sensitive to different datasets~\cite{DBLP:conf/aaai/0001SPZY23} and may yield suboptimal performance~\cite{DBLP:conf/aaai/YinWHXZ22}.

To achieve better generalization and globally optimal performance, 
learnable graph augmentation is proposed to automatically disrupt redundant information as much as possible to share minimal yet sufficient core information between augmented views~\cite{DBLP:conf/nips/TongLDDLW21,suresh2021adversarial}. Although they have achieved great success, there still remain two open challenges worth exploring. (1) Community structure plays a crucial role in various downstream tasks, such as node classification and link prediction~\cite{DBLP:conf/www/LiJT22,DBLP:conf/ijcai/ChenZLZ0Z23}.
However, current GCL methods often randomly disrupt graphs during graph augmentation, which inevitably leads to the corruption of high-level graph information (i.e., community) and limits the generalization~\cite{DBLP:conf/focs/ChiplunkarKKMP18}.
(2) Current constraints employed in learnable graph augmentation methods primarily focus either on topology or node features~\cite{DBLP:conf/icml/LiWZW0C22}. For instance, GAME~\cite{DBLP:conf/icml/WeiWBNBF23} and GCL-SPAN~\cite{lin2023spectral} use spectrum-based constraints for topology augmentation. Due to the asymmetry of the feature matrix, their methods cannot be extended to feature augmentation.
On the other hand, COSTA~\cite{DBLP:conf/kdd/ZhangZSKK22} designs a covariance-preserving constraint for feature augmentation, which, however, lacks effectiveness in topology augmentation.
By solely focusing on one type of graph augmentation (topology or feature), models may not fully exploit all available information and lack robustness against different types of noise~\cite{DBLP:conf/icml/LiuYDLXRZHW22}.

To solve the aforementioned issues, we propose a general learnable \underline{C}ommunity-\underline{I}nvariant \underline{GCL} framework (\textbf{CI-GCL}), which unifies constraints from both topology and feature augmentation to maintain CI for learnable graph augmentation.
Specifically, when considering topology augmentation with a certain degree of disruption, we observe a nearly negative correlation between community and spectral changes (see Sec~\ref{pre-experiment}).
Therefore, to maximize the topology perturbation while ensuring community invariance, we can simply maximize graph spectral changes during topology augmentation.
%
To extend our CI constraint to feature augmentation, we convert the feature matrix into a symmetric bipartite feature matrix based on the bipartite graph co-clustering technique~\cite{DBLP:journals/pami/ZhangNL23}. This approach converts feature augmentation into bipartite feature augmentation, while elucidating the importance of features in maintaining community structure.
%
For bipartite feature augmentation, we also observed a negative relationship between community and spectral changes, which is consistent with topology augmentation. This motivates us to apply our CI constraint to feature augmentation by maximizing graph spectral changes during bipartite feature augmentation. To summarize, 
the contributions of this research are:
\begin{itemize}
    \item We propose a learnable CI-GCL framework to automatically maintain CI during graph augmentation by maximizing spectral change loss, improving the model's downstream performances.
    \item We theoretically show that the proposed CI constraint can be applied to both topology and feature augmentation, enhancing the model's robustness.
    \item Experiments on 21 widely used benchmarks demonstrate the effectiveness and robustness of CI-GCL.
\end{itemize}

\begin{table}[h]
\centering
\scriptsize
\caption{Graph augmentation (Aug.) method comparison. An ideal method should support both topology and node feature augmentation, be adaptive to different datasets, be end-to-end differentiable and have efficient back-propagation (BP), and be CI and have unified constraints for any augmentation to against various noise.}
\setlength{\tabcolsep}{0.6mm}{
\begin{tabular}{@{}lcccccccc@{}}
\toprule
\multirow{2}{*}{Property} & \multicolumn{3}{c}{Random or Constraint} &  & \multicolumn{4}{c}{Learnable Graph Augmentation}        \\ \cmidrule(lr){2-4} \cmidrule(l){6-9} 
& GraphCL    & JOAO    & GCA    &  & AutoGCL & AD-GCL  & GCL-SPAN & Ours \\ 
                               \midrule
Topology Aug.             & \lidy{\ding{51}}             & \lidy{\ding{51}}          & \lidy{\ding{51}}            &  & \lidy{\ding{51}}        & \lidy{\ding{51}}       & \lidy{\ding{51}}    & \lidy{\ding{51}}             \\
Feature Aug.              & \lidy{\ding{51}}             & \lidy{\ding{51}}          & \lidy{\ding{51}}             &  & \lidy{\ding{51}}        & -      & -   & \lidy{\ding{51}}            \\
Adaptive                  & -             & \lidy{\ding{51}}          & \lidy{\ding{51}}            &  & \lidy{\ding{51}}        & \lidy{\ding{51}}       & \lidy{\ding{51}}    & \lidy{\ding{51}}             \\
Differentiable            & -             & -          & -             &  & \lidy{\ding{51}}        & \lidy{\ding{51}}       & -   & \lidy{\ding{51}}             \\
Efficient BP              & -             & -          & -             &  & \lidy{\ding{51}}        & \lidy{\ding{51}}       & -   & \lidy{\ding{51}}           \\
Community      & -             & -          & -             &  & -       & -      & -   & \lidy{\ding{51}}            \\
Unified Constraint        & -             & -          & -             &  & -       & -      & -   & \lidy{\ding{51}}             \\ \bottomrule
\end{tabular}}\label{tab: summary}
\end{table}

\section{Related Work}
\label{related_work}

As an effective self-supervised learning paradigm, contrastive learning has achieved great success to learn text or image representations~\cite{DBLP:conf/icml/ChenK0H20,DBLP:conf/nips/ZhangZLZH20}.
DGI~\cite{DBLP:conf/iclr/VelickovicFHLBH19} first adopted contrastive learning to 
learn robust graph representations that are invariant to various noise and operations.
However, different from Euclidean or sequential data, graphs are irregular non-Euclidean data and sensitive to minor structural augmentation~\cite{DBLP:conf/aaai/0001SPZY23}, resulting in learned graph representations being ineffective.
Given that among many GCL variations, graph augmentation shows its crucial advantage for graph representation learning, many studies attempted to investigate effective graph augmentation for GCL.
Prior GCL almost adopts random graph augmentation.
For example, 
GRACE~\cite{zhu2020deep} firstly 
uses random edge dropping and feature masking as graph augmentations. 
After that, GraphCL~\cite{DBLP:conf/nips/YouCSCWS20} gives an extensive study on different combinations of graph augmentations including randomly node dropping, edge perturbation, subgraph sampling, and feature masking. 
To make GraphCL more flexible, JOAO~\cite{DBLP:conf/icml/YouCSW21} automatically selects the combination of different random graph augmentations.
Due to the limited generalization of random augmentation, researchers start to incorporate expert knowledge as constraints for graph augmentation.
For instance, 
\citet{DBLP:conf/icml/DuanVPRM22} and GraphAug~\cite{DBLP:conf/iclr/LuoMAKUMJ23} employ label-invariance between original and augmented views as constraints, which achieves great success in the graph-level classification task.

Recent GCL focuses on fully parameterizing graph augmentation for utilizing learnable graph augmentation to automatically determine how to disrupt 
graphs~\cite{DBLP:conf/www/ChenZ0TX23}.
For example,
AutoGCL~\cite{DBLP:conf/aaai/YinWHXZ22} and AD-GCL~\cite{suresh2021adversarial} build a learnable graph generator that learns a probability distribution to help adaptively drop nodes and mask features.
CGI~\cite{wei2022contrastive} introduces the Information Bottleneck theory into GCL 
to remove unimportant nodes and edges between two augmented graphs by minimizing shared mutual information.
GAME~\cite{DBLP:conf/nips/Liu0BSP22} and GCL-SPAN~\cite{lin2023spectral} explore graph augmentation in spectral space, by maximizing spectral changes of high-frequency 
or all components to automatically drop edges.
AdaGCL~\cite{DBLP:conf/kdd/Jiang0H23} and GACN~\cite{DBLP:conf/kdd/WuWXLZWSG23} design graph generators and discriminators to automatically augment graphs in an adversarial style.

Compared with previous studies, we are the first to point out the importance of community invariance for graph augmentation and propose a unified CI constraint for both topology and feature augmentation by simply maximizing spectral changes.
Detailed comparisons are listed in Table~\ref{tab: summary}. 

\section{Preliminary}

Let $G=(\mathbf{X}, \mathbf{A})$ be a graph with $n$ nodes and $m$ edges, 
where $\mathbf{X} \in \mathbb{R}^{n \times d}$ describes node features 
and $\mathbf{A} \in \{0,1\}^{n \times n}$  denotes an adjacency matrix 
with $\mathbf{A}_{i j}=1$ if an edge exists between node $i$ and $j$, otherwise $\mathbf{A}_{i j}=0$. 
The normalized Laplacian matrix is defined as $\mathbf{L}_{\text{norm}} = \text{Lap}(\mathbf{A})=\mathbf{I}_{n}-\mathbf{D}^{-1 / 2} \mathbf{A} \mathbf{D}^{-1 / 2}$, where $\mathbf{I}_{n} \in \mathbb{R}^{n \times n}$ is an identity matrix, $\mathbf{D}=\text{diag}\left(\mathbf{A} \mathbf{1}_{n}\right)$ is the diagonal degree matrix with $\mathbf{1}_{n} \in \mathbb{R}^{n}$ being an all-one vector.

\noindent \textbf{Graph Spectrum.} The spectral decomposition of $\mathbf{L}_{\text{norm}}$ is defined as $\mathbf{L}_{\text {norm }}= \text{Lap}(\mathbf{A}) = \mathbf{U} \boldsymbol{\Lambda} \mathbf{U}^{\top}$, where 
the diagonal matrix $\boldsymbol{\Lambda}=\text{eig}\left(\text{Lap}\left(\mathbf{A}\right)\right)= \text{diag}\left(\lambda_{1}, \ldots, \lambda_{n}\right)$ consists of real eigenvalues known as graph spectrum, and $\mathbf{U}=\left[\mathbf{u}_{1}, \ldots, \mathbf{u}_{n}\right] \in \mathbb{R}^{n \times n}$ are the corresponding orthonormal eigenvectors known as the spectral bases~\cite{golub2013matrix}.

\noindent \textbf{Graph Representation Learning.}
Let $\mathcal{G}$ denote the whole graph space with $G \in \mathcal{G}$.
Graph representation learning aims to train an encoder $f_{\theta}(\cdot): \mathcal{G} \to \mathbb{R}^{n \times d'}$ to obtain node representations. 
%
Then, it trains a readout function $r_{\phi}(\cdot): \mathbb{R}^{n \times d'} \to \mathbb{R}^{d'}$ by pooling all node representations to obtain a low-dimensional vector for graph $G$, which can be used in graph-level tasks.

\noindent \textbf{Graph Contrastive Learning.} 
GCL trains the encoder $f_{\theta}(\cdot)$ to capture the maximum mutual information between the original graph and its perturbed view by graph augmentation.
Formally, letting 
$T_{1}(G)$ and $T_{2}(G)$ denote two graph augmentation distributions of $G$, GCL is defined as follows: 
\begin{equation}
    \mathop{\min}_{\theta,\phi}\,\, \mathcal{L}_{\text{GCL}}(t_{1}(G), \, t_{2}(G), \, \theta,\phi),
\end{equation}
where $t_{m}(G) \sim T_{m}(G)$ with $m \in \{1,2\}$ and
$\mathcal{L}_{\text{GCL}}$ measure the disagreement between two augmented graphs.

\section{Methodology}\label{Sec:methodology}
We first show the importance of community invariance in GCL with preliminary analysis. Then, we introduce the details of our methodology CI-GCL as illustrated by Figure~\ref{fig:overall}. 

\subsection{Preliminary Analysis}\label{pre-experiment}
%

\begin{figure}[h]
	\centering
	\includegraphics[scale=0.248]{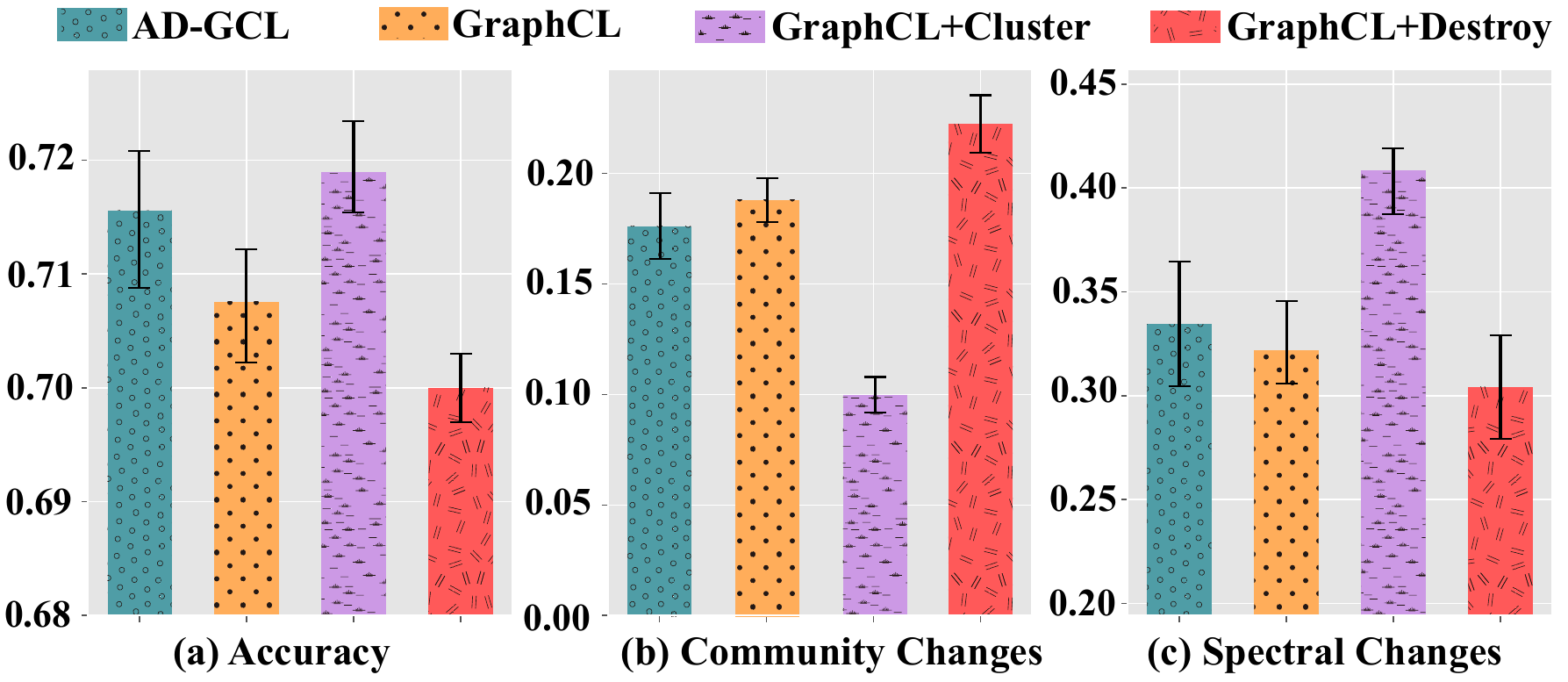}
	\caption{In unsupervised graph classification,
 we define community changes as the average ratio of the changed community labels over the number of nodes before and after graph augmentation by spectral clustering. Spectral changes are the eigenvalue changes between original and augmented graphs, using the $L_{2}$ distance.}
    \label{fig:pre}
\end{figure}


Preserving community structure is crucial for learnable graph augmentation, i.e.,  perturbing a constrained number of edges or features that have least impact to community changes of the input graph. To show the benefits of preserving communities, we conduct a preliminary experiment by applying GraphCL (unlearnable graph augmentation) and AD-GCL (learnable graph augmentation) on the IMDB-B dataset. Specifically, we design the following four methods: (1) AD-GCL with uniformly edge dropping; (2) GraphCL with uniformly edge dropping; (3) GraphCL+Cluster augmentation that removes edges between different clusters with a higher probability; (4) GraphCL+Destroy augmentation that removes edges within the same cluster with a higher probability. Note that (3) preserves community structure, while (4) tends to disrupt community structure, as indicated by recent studies~\cite{DBLP:conf/focs/ChiplunkarKKMP18,10.1145/3534678.3539435}.
We plot the accuracy for unsupervised graph classification as Figure~\ref{fig:pre}(a) and the community changes as (b). From Figure~\ref{fig:pre}(a-b), we can observe: (1) Methods' performance generally exhibits a nearly negative correlation with their community changes (i.e., less cluster changes yield higher accuracy); (2) AD-GCL outperforms GraphCL but underperforms GraphCL+Cluster. All these indicate preserving the community structure yields better results. Moreover, we also draw the spectral changes as Figure~\ref{fig:pre}(c), as graph spectrum can reflect high-level graph structural information~\cite{spielman2012spectral,DBLP:journals/jacm/LeeGT14}. Through Figure~\ref{fig:pre}(b-c), we can see that spectral changes are almost negatively correlated with community changes. That is to say, we can preserve community invariance during graph augmentation by maximizing spectral changes, based on which we expand on our methodology as follows.



\begin{figure*}[t]
\centering
\includegraphics[scale=0.52]{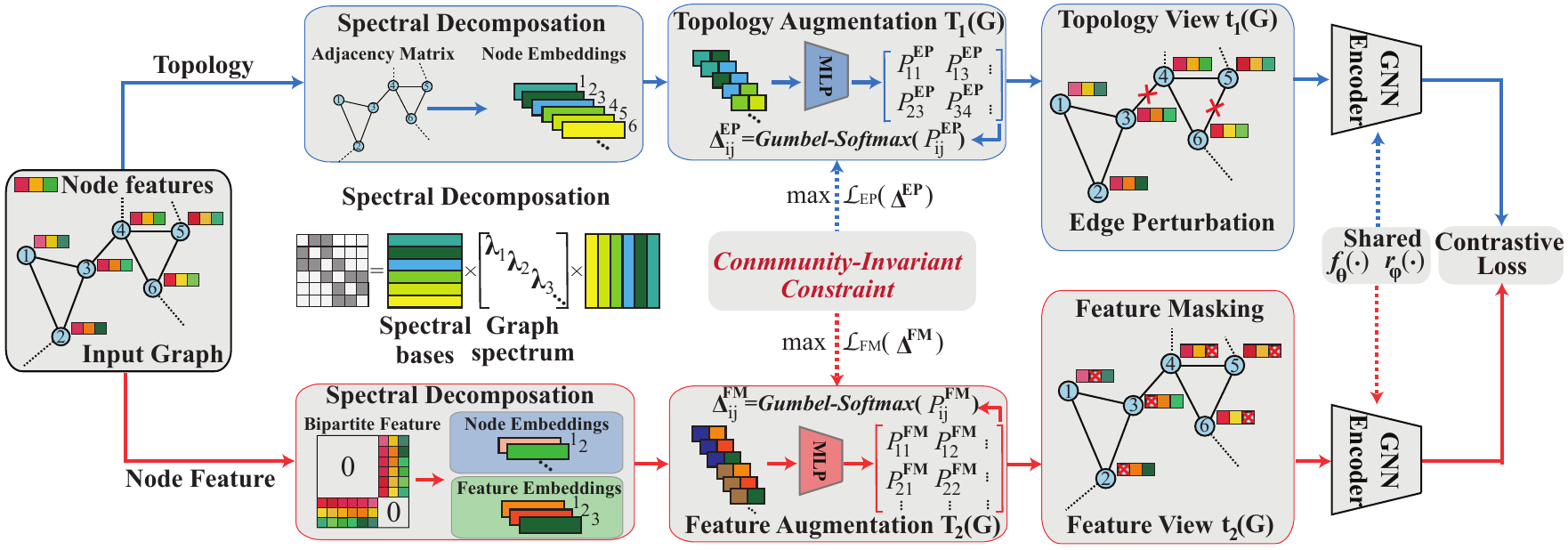}
\caption{The proposed CI-GCL consists of two core components:
(1) Learnable graph augmenter optimizes $T_{m}(G)$ to disrupt redundant information while ensuring community invariance from the original graph.
(2) The GNN encoder $f_{\theta}(\cdot)$ and Readout $r_{\phi}(\cdot)$ maximize the mutual information between two augmented graphs by contrastive loss. 
We use edge dropping and feature masking as an instantiation.}
\label{fig:overall}
\end{figure*}




\subsection{Community-Invariant Graph Augmentation}
\label{sec:spectral}

\textbf{Topology Augmentation.} 
We conduct edge perturbation and node dropping operations as our topology augmentation.
For edge perturbation, we define $T_{1}(G)$ as a Bernoulli distribution $Bern(P^{\text{EP}}_{ij})$ for each edge $A_{ij}$.
Then, we can sample the edge perturbation matrices $\mathbf{\Delta}^{\text{EP}} \in \{0,1\}^{n \times n}$, where $\Delta_{ij} \sim Bern(P^{\text{EP}}_{ij})$ indicates whether to flip the edge between nodes $i$ and $j$. 
The edge is flipped if $\Delta^{\text{EP}}_{ij}=1$; otherwise, it remains unchanged. A sampled augmented graph by edge augmentation can be formulated as:
\begin{equation}\label{T1}
    t_{1}^{\text{EP}}(G) = \mathbf{A} + \mathbf{C} \circ \mathbf{\Delta^{\text{EP}}},\,\,\mathbf{C} = \mathbf{A^{c}} - \mathbf{A},
\end{equation}
where $\circ$ denotes an element-wise product and $\mathbf{A^{c}}$ represents the complement matrix of $\mathbf{A}$, calculated as $\mathbf{A^{c}}=\mathbf{J}-\mathbf{I}_{n}-\mathbf{A}$, where $\mathbf{J}$ denotes an all-one matrix. 
Thus, $\mathbf{C} \in \{-1,1\}^{n \times n}$ denotes all edge flipping operations, i.e., an edge is added between nodes $i$ and $j$ if $C_{ij}=1$, and removed if $C_{ij}=-1$. 

However, Eq.(\ref{T1}) cannot be directly applied to learnable graph augmentation, since Bernoulli sampling is non-differentiable. 
Inspired by \citet{jang2017categorical}, we soften it from the discrete Bernoulli distribution space to the continuous space with a range $\Delta^{\text{EP}}_{ij}\in (0,1)^{n \times n}$ using \textit{Gumbel-Softmax}, which can be formulated as:
%
\begin{align}\label{eq:node_dropping_probability}
\Delta^{\text{EP}}_{ij}(\epsilon) &=\textrm{Softmax}((\log (P^{\text{EP}}_{ij})+\epsilon)/\tau),\\
P^{\text{EP}}_{ij}&= \text{Sigmoid} (\text{MLPs} (\text{Concat}(\mathbf{e}_{i},\mathbf{e}_{j}))),
\end{align}
where $P^{\text{EP}}_{ij}$ controls whether to flip edge $A_{ij}$, MLPs are multilayer perceptions,  
$\mathbf{e}_{i}$ is the $i$-th node representation, $\text{Concat}(\cdot , \cdot)$ denotes the concatenation operation, $\epsilon \sim {Gumbel}(0,1)$,\footnote{The $Gumbel(0,1)$ distribution can be sampled by calculating $\epsilon=-\text{log}(-\text{log}(u))$ with $u \sim \text{Uniform}(0,1)$~\cite{jang2017categorical}.}  
and $\tau$ is the temperature factor that controls the approximation degree for the discrete categorical distributions. 
$\tau > 0$ results in a well-defined gradient $\partial \Delta^{\text{EP}}_{ij}/\partial P^{\text{EP}}_{ij}$, facilitating efficient optimization.

Node dropping can be considered as a type of edge dropping, i.e., removing all edges connected to this node. Thus, node dropping can be formulated similarly to Eq.(\ref{T1}) as: 
\begin{equation}\label{eq:T2}
 t_{1}^{\text{ND}}(G) = \mathbf{A} + \mathbf{(-A)} \circ \mathbf{\Delta}^{\text{ND}}.   
\end{equation}
where $\mathbf{\Delta}^{\text{ND}}$ can be calculated by:
\begin{align}
\mathbf{\Delta}^{\text{ND}} &= (\mathbf{\Psi}^{\text{ND}} \cdot \textbf{1}_{n}^{\top} + (\mathbf{\Psi}^{\text{ND}} \cdot \textbf{1}_{n}^{\top})^{\top})/2, \\
    \Psi^{\text{ND}}_{i}(\epsilon) &= \textrm{Softmax}((\log(P^{\text{ND}}_{i}) + \epsilon)/\tau),  \\  
   P^{\text{ND}}_{i} &= \textrm{Sigmoid}(\text{MLPs}(\mathbf{e}_{i})).
\end{align}
We combine $t_{1}^{(\text{EP},\text{ND})}(G)$ as topology augmentation $t_{1}(G)$.

\textbf{CI-based Topology Augmentation.}
Inspired by the findings in Sec~\ref{pre-experiment}, 
we aim to optimize $\mathbf{\Delta}$ by simultaneously maximizing graph disruption while minimizing community changes for learnable topology augmentation in Eqs.(\ref{T1},\ref{eq:T2}).
Based on the matrix perturbation theory~\cite{DBLP:conf/icml/BojchevskiG19}, we have the following definition.

\begin{definition}\label{eigenvalue}
Let $\lambda_{k}$ denote the $k$-th eigenvalue of the spectral decomposition of $\textrm{Lap}(\mathbf{A}) =  \mathbf{U} \Lambda \mathbf{U}^{\top}$. 
For a single edge perturbation $A_{ij}$, it induces absolute changes in eigenvalues given by $\sum_{k=1}^{n}  |\Delta \lambda_{k}|   = \sum_{k=1}^{n} | (U_{i k}-U_{j k})^{2} +(\lambda_{k}-1)(U_{i k}^{2}+U_{j k}^{2})|$, and $\Delta \lambda_{k}$ denotes the $k$-th spectral change.  
\end{definition}

When optimizing $\sum_{k=1}^{n}  |\Delta \lambda_{k}| $ in Definition~\ref{eigenvalue}, we argue that 
maintaining community invariance requires a categorized discussion on edge adding and edge dropping.

\begin{theorem}\label{the:abslute}
The absolute spectral changes $\sum_{k=1}^{n}  |\Delta \lambda_{k}|$ are upper bounded by $\left\|\mathbf{U}_{i \cdot}-\mathbf{U}_{ j \cdot}\right\|^{2}_{2} + \sum_{k=1}^{n} |\lambda_{k}-1| $ and lower bounded by $\left\|\mathbf{U}_{ i \cdot}-\mathbf{U}_{j \cdot}\right\|^{2}_{2}  - \sum_{k=1}^{n} |\lambda_{k}-1|$, respectively. 
Here, $\mathbf{U}_{i \cdot}$ represents the $i$-th row vector of $\mathbf{U}$, denoting the $i$-th node embedding in the spectral space.
\end{theorem}

According to Theorem~\ref{the:abslute}, maximizing spectral changes equates to maximizing their upper bound, i.e., flipping several edges between nodes with largest distances in spectral space. 
\citet{DBLP:conf/cvpr/ShiM97} states that node representations with larger distances always belong to different communities.
Thus, we can maximize spectral changes during the edge dropping to preserve community invariance.
However, we cannot conduct edge adding since adding edges between clusters always disrupts communities~\cite{DBLP:journals/tfs/ZhuZYN23}.
Conversely, 
minimizing spectral changes equates to minimizing their lower bound, i.e., flipping several edges between nodes with lowest distances, where nodes with lower distances always belong to the same cluster.
Thus, we can minimize spectral changes during the edge adding, instead of edge dropping, since dropping edges within one cluster will disrupt communities~\cite{DBLP:journals/tfs/ZhuZYN23}.
We formulate the CI constraint for edge perturbation by jointly optimizing edge dropping and adding as follows:
\begin{align}
&\max\limits_{\mathbf{\Delta}^{\text{ED}},\, 
\mathbf{\Delta}^{\text{EA}} \,\in \mathcal{S}}  \mathcal{L}_{\text{EP}}(\mathbf{\Delta}^{\text{EP}}) = \mathcal{L}_{\text{ED}}(\mathbf{\Delta}^{\text{ED}})\, - \,
\mathcal{L}_{\text{EA}}(\mathbf{\Delta}^{\text{EA}}), \label{over-2} \\
&\mathcal{L}_{\text{ED}}(\mathbf{\Delta}^{\text{ED}}) = \|\text{eig}(\text{Lap}(\mathbf{A}-\mathbf{A} \circ \mathbf{\Delta}^{\text{ED}}))-\text{eig}(\text{Lap}(\mathbf{A}))\|_{2}^{2}, \nonumber \\
&\mathcal{L}_{\text{EA}}(\mathbf{\Delta}^{\text{EA}}) = \|\text{eig}(\text{Lap}(\mathbf{A}+\mathbf{A^{c}} \circ \mathbf{\Delta}^{\text{EA}}))-\text{eig}(\text{Lap}(\mathbf{A}))\|_{2}^{2},\nonumber
\end{align}
where 
$\mathcal{S} = \{\mathbf{S}|\mathbf{S} \in [0,1]^{n \times n}, \|\mathbf{S}\|_{1} \leq \psi \}$, $\psi$ controls the perturbation strength, and $\mathcal{L}_(\mathbf{\Delta})$ represents graph spectral changes under different augmented operations.

Node dropping can be considered as one type of ED, which can be constrained by community invariance by:
\begin{equation}\label{over-3}
\max\limits_{\mathbf{\Delta}^{\text{ND}} \in \mathcal{S}} \|\text{eig}(\text{Lap}(\mathbf{A}-\mathbf{A} \circ \mathbf{\Delta}^{\text{ND}}))-\text{eig}(\text{Lap}(\mathbf{A}))\|_{2}^{2}.
\end{equation}
By jointly optimizing Eqs.(\ref{over-2},\ref{over-3}), we can maximize topology perturbation while maintaining community invariance.

\textbf{Feature Augmentation.} 
Similar to topology augmentation, we define $T_{2}(G)$ as a Bernoulli distribution $Bern(P^{\text{FM}}_{ij})$ for each feature $X_{ij}$.
Then, we can sample feature masking matrix $\Delta_{ij}^{\text{FM}} \sim Bern(P^{\text{FM}}_{ij})$, indicating whether to mask the corresponding feature. A sampled augmented graph by feature masking can be formulated as:
\begin{equation}\label{T3}
    t_{2}^{\text{FM}}(G) = \mathbf{X} + (-\mathbf{X}) \circ \mathbf{\Delta}^{\text{FM}}.
\end{equation}

\textbf{CI-based Feature Augmentation.}
Different from topology augmentation, $\mathbf{X} \in \mathbb{R}^{n \times d}$ is an asymmetric matrix lacking spectral decomposition. Theorem~\ref{the:abslute} is not applicable to feature augmentation. 
Moreover, discerning which feature has the least impact on community changes is challenging.
Inspired by the co-clustering of bipartite graph~\cite{DBLP:conf/nips/NieWDH17}, which can determine the importance of features for node clustering, we construct the feature bipartite graph as:
\begin{equation}
\widetilde{\mathbf{X}} = 
\begin{bmatrix}
    \mathbf{0} & \mathbf{X} \\
    \mathbf{X}^{\top} & \mathbf{0} \\
\end{bmatrix}_{(n+d)\times (n+d),}
\end{equation}
where the first $n$ rows of $\widetilde{\mathbf{X}}$ denote the original nodes, while the subsequent $d$ rows serve to represent features as feature nodes.
Then, $\widetilde{X}_{ij}$, where $i \in \{1,\cdots n\}$ and $j \in \{(n+1),\cdots, (n+d)\}$, can be interpreted as the linking weight 
between $i$-th node with $j$-th feature.

\begin{theorem}\label{the:2}
Let the singular value decomposition of the feature matrix $\mathbf{X}$ be denoted as $\text{svd}(\mathbf{D}_{u}^{-1/2}\mathbf{X}\mathbf{D}_{v}^{-1/2}) = \mathbf{U} \mathbf{\Lambda}_{1} \mathbf{V}^{\top}$ where $\mathbf{D}_{u}$ and $\mathbf{D}_{v}$ are the degree matrices of $\mathbf{X}$ and $\mathbf{X}^{\top}$, and $\mathbf{U}$ and $\mathbf{V}$ represent the left and right singular vectors, respectively. 
Then, $\text{eig}(\text{Lap}(\widetilde{\mathbf{X}})) = \mathbf{F} \mathbf{\Lambda}_{2} \mathbf{F}^{\top}$ where the $k$-th smallest eigenvector $\mathbf{F}_{\cdot k}$ is equal to the concatenation of $k$-th largest singular vectors: $\mathbf{F}_{\cdot k}=[\mathbf{U}_{\cdot k}; \mathbf{V}_{\cdot k}]$. 
\end{theorem}

According to Theorem~\ref{the:2} and findings from \citet{DBLP:conf/nips/NieWDH17}, if we can maintain community invariance of $\widetilde{\mathbf{X}}$, community structure will also be preserved in $\mathbf{X}$. Hence, we investigate the community-invariant constraint in $\widetilde{\mathbf{X}}$.

\begin{theorem}\label{the:3}
Let $\lambda_{k}$ denote the $k$-th smallest eigenvalue of $\mathbf{\Lambda_{2}}$.
When masking one feature $\widetilde{X}_{ij}$, the induced spectral changes are given by $\widetilde{\mathbf{X}}$ as $\sum_{k=1}^{n+d}  |\Delta \lambda_{k}|$   = $\sum_{k=1}^{n+d}$  $| (F_{ik}  -F_{jk} )^{2} +(\lambda_{k}-1)(F_{ik}^{2}+F_{jk}^{2})|$, which are upper bounded by $\|{\mathbf{F}}_{i \cdot} - {\mathbf{F}}_{j \cdot}\|_{2}^{2} + \sum_{k=1}^{n+d} |1-\lambda_{k}|$ where $i \in \{1,..,n\}$ and $j \in \{(n+1),..,(n+d)\}$, ${\mathbf{F}}_{i\cdot}$ is the $i$-th row vector of ${\mathbf{F}}$.
\end{theorem}

Based on Theorem~\ref{the:3},
maximizing spectral changes in $\widetilde{\mathbf{X}}$ under a constrained number of perturbation equals finding several largest embedding distances between nodes and feature nodes, i.e., these features have the least impact on community changes for these nodes~\cite{DBLP:journals/pami/ZhangNL23}. 
Thus, CI constraint for feature augmentation
$\mathcal{L}_{\text{FM}}(\mathbf{\Delta}^{\text{FM}})$ can be formulated as follows: 
\begin{equation}\label{over-4}
\max\limits_{{\mathbf{\Delta}^{\text{FM}}} \in \mathcal{S}} 
\| \text{eig}(\text{Lap}(\widetilde{\mathbf{X}}-\widetilde{\mathbf{X}} \circ {\mathbf{\Delta}^{\text{FM}}})-\text{eig}(\text{Lap}(\widetilde{\mathbf{X}}))) \|_{2}^{2}. 
\end{equation}

Finally, we parameterize $\mathbf{\Delta}^{\text{FM}}$ and ensure its differentiability in the feature augmentation, formulated as follows:
\begin{align}\label{eq:feature_masking_probability}
\Delta^{\text{FM}}_{ij}(\epsilon) &=\textrm{Softmax}((\log(P^{\text{FM}}_{ij})+\epsilon)/\tau),\\
P^{\text{FM}}_{ij}&= \text{Sigmoid}(\text{MLPs}(\text{Concat}(\widetilde{\mathbf{U}}_{\cdot i},\widetilde{\mathbf{V}}_{\cdot j })))_{.} \nonumber
\end{align}



\noindent\textbf{CI-GCL.} As shown in Figure~\ref{fig:overall}, we instantiate a graph contrastive learning framework with the proposed community-invariant constraint, namely CI-GCL. 
Specifically, we first conduct spectral decomposition on the adjacent matrix and feature bipartite matrix to obtain node and feature representations.
Then, we consider these node and feature representations as input of MLPs for
both topology and feature augmentation, where we randomly initialize the parameters of MLPs.
For each iteration of contrastive learning, we sample two augmented graphs by topology augmentation 
and feature augmentation. 
The augmented graphs are then fed into a GCN encoder $f_{\theta}(\cdot)$, which outputs two sets of node representations.
A readout pooling function $r_{\phi}(\cdot)$ is applied to aggregate and transform the node representations and obtain graph representations $\mathbf{z}^{(1)}, \mathbf{z}^{(2)}$.
Following GraphCL~\cite{DBLP:conf/nips/YouCSCWS20}, given training graphs $\mathcal{G}$, we use contrastive objective $\mathcal{L}_{\text{GCL}}$, which can be defined as: 
\begin{align}\label{GCL-function}
&\min\limits_{\theta,\phi} \,\, \mathcal{L}_{\text{GCL}}\left(t_{1}\left(G\right),t_{2}\left(G\right),\theta,\phi\right) \\
= -\frac{1}{|\mathcal{G}|}&\sum_{n=1}^{|\mathcal{G}|} \left( 
\log \frac{\exp (\text{sim}(\mathbf{z}_{n}^{(1)},\mathbf{z}_{n}^{(2)}/\tau_{2} ))}{\sum_{n'=1,n'\neq n}^{|\mathcal{G}|}\exp (\text{sim}(\mathbf{z}_{n}^{(1)},\mathbf{z}_{n'}^{(2)})/\tau_{2})}
\right) \nonumber,
\end{align}
where $\tau_{2}$ is the temperature parameter, and we conduct minibatch optimization for Eq.(\ref{GCL-function}) in our study.
\begin{table*}[t]
\centering
\footnotesize
\caption{Unsupervised representation learning classification accuracy (\%) on TU Datasets. \textbf{Bold} denotes the best performance, and \underline{underline} represents the second best performance. \textcolor{blue}{\ding{101}} marks the reproduced results of the corresponding baselines by us. }
\setlength{\tabcolsep}{1.2mm}{ 
\begin{tabular}{l|cccccccc|c}
\toprule
\rowcolor{gray!30}
\textbf{Method} &
{\textbf{NCI1} $\uparrow$} & {\textbf{PROTEINS} $\uparrow$} & {\textbf{DD} $\uparrow$} & {\textbf{MUTAG} $\uparrow$} & 
{\textbf{COLLAB} $\uparrow$} & {\textbf{RDT-B} $\uparrow$} & {\textbf{RDT-M5K} $\uparrow$} & {\textbf{IMDB-B} $\uparrow$}  & {\textbf{Avg.} $\uparrow$}\\
\midrule
\midrule
InfoGraph & 76.20$\pm$1.0 & 74.44$\pm$0.3 & 72.85$\pm$1.7 & 89.01$\pm$1.1 & 70.65$\pm$1.1 & 82.50$\pm$1.4 & 53.46$\pm$1.0 & 73.03$\pm$0.8 & 74.02 \\
GraphCL   & 77.87$\pm$0.4 & 74.39$\pm$0.4 & 78.62$\pm$0.4 & 86.80$\pm$1.3 & 71.36$\pm$1.1 & 89.53$\pm$0.8 & 55.99$\pm$0.3 & 71.14$\pm$0.4 & 75.71 \\
MVGRL     & 76.64$\pm$0.3 & 74.02$\pm$0.3        & 75.20$\pm$0.4 & 75.40$\pm$7.8 & 73.10$\pm$0.6 & 82.00$\pm$1.1 & 51.87$\pm$0.6 & 63.60$\pm$4.2 & 71.48 \\
JOAO      & 78.07$\pm$0.4 & 74.55$\pm$0.4 & 77.32$\pm$0.5 & 87.35$\pm$1.0 & 69.50$\pm$0.3 & 85.29$\pm$1.4 & 55.74$\pm$0.6 &  70.21$\pm$3.0  & 74.75\\
SEGA      &  \underline{79.00$\pm$0.7} & \underline{76.01$\pm$0.4} & 78.76$\pm$0.6 & \textbf{90.21$\pm$0.7} & \underline{74.12$\pm$0.5} &  \underline{90.21$\pm$0.7} & 56.13$\pm$0.3 & \underline{73.58$\pm$0.4} & \underline{77.25} \\
GCS\textcolor{blue}{\ding{101}}      & 77.18$\pm$0.3 & 74.04$\pm$0.4 & 76.28$\pm$0.3 & 88.19$\pm$0.9 & 74.00$\pm$0.4 & 86.50$\pm$0.3 & \underline{56.30$\pm$0.3} & 72.90$\pm$0.5  & 75.64\\
GCL-SPAN\textcolor{blue}{\ding{101}}  & 75.43$\pm$0.4 & 75.78$\pm$0.4 & \underline{78.78$\pm$0.5} & 85.00$\pm$0.8 & 71.40$\pm$0.5 & 86.50$\pm$0.1 & 54.10$\pm$0.5 & 66.00$\pm$0.7  & 74.12 \\
AD-GCL\textcolor{blue}{\ding{101}}   & 73.38$\pm$0.5 & 73.59$\pm$0.7 & 75.10$\pm$0.4 & \underline{89.70$\pm$1.0} & 72.50$\pm$0.6 & 85.52$\pm$0.8 & 54.91$\pm$0.4 & 71.50$\pm$0.6 & 74.53 \\
AutoGCL\textcolor{blue}{\ding{101}}   & 78.32$\pm$0.5 & 69.73$\pm$0.4 & 75.75$\pm$0.6 & 85.15$\pm$1.1 & 71.40$\pm$0.7 & 86.60$\pm$1.5 & 55.71$\pm$0.2 & 72.00$\pm$0.4 &74.33  \\
\midrule
CI+AD-GCL     & 74.35$\pm$0.5 & 74.66$\pm$0.6 & 76.20$\pm$0.4 & 89.88$\pm$0.7 & 73.94$\pm$0.3 & 87.80$\pm$1.2 & 54.75$\pm$0.6 & 72.10$\pm$0.3 &75.46\\
CI+AutoGCL   & 78.47$\pm$0.7 & 70.81$\pm$0.5 & 76.53$\pm$0.6 & 86.73$\pm$1.0 & 72.24$\pm$0.9 & 87.50$\pm$1.4 & 55.97$\pm$0.2 & 72.50$\pm$0.3  &75.09  \\
\midrule
\textbf{CI-GCL} & \textbf{80.50$\pm$0.5} & \textbf{76.50$\pm$0.1} & \textbf{79.63$\pm$0.3} & 89.67$\pm$0.9 & \textbf{74.40$\pm$0.6} & \textbf{90.80$\pm$0.5} & \textbf{56.57$\pm$0.3} & \textbf{73.85$\pm$0.8} &\textbf{77.74} \\
\bottomrule
\end{tabular}}
\label{tab:unsup_acc-1}
\end{table*}
\subsection{Optimization and Scalability}\label{sec:optimization}
\textbf{Optimization.}
Eqs.(\ref{over-2},\ref{over-3},\ref{over-4}) are jointly optimized via projected gradient descent.
Taking $\mathbf{\Delta}^{\text{FM}}$ in Eq.(\ref{over-4}) as an example, we can update the parameters $\mathbf{\Delta}^{\text{FM}}$ as:
\begin{equation}\label{update}
    \mathbf{\Delta}^{\text{FM}}_{t} = \mathcal{P}_{\mathcal{S}}\left(\mathbf{\Delta}^{\text{FM}}_{(t-1)}-\eta_{t} \nabla \mathcal{L}_{\text{FM}}\left(\mathbf{\Delta}^{\text{FM}}_{(t-1)}\right)\right),
\end{equation}
where $\mathcal{P}_{\mathcal{S}}(\mathbf{\Delta}) = \arg\min_{\mathbf{S}\in \mathcal{S}}\|\mathbf{S}-\mathbf{\Delta}\|_{F}^{2}$ is defined as one projection operation at $\mathbf{\Delta}$ over the constraint set $\mathcal{S}$ and $\eta_{t}>0$ is the learning rate for the $t$-th updating step.
The gradient $\nabla \mathcal{L}_{\text{FM}}(\mathbf{\Delta}^{\text{FM}}_{(t-1)})$ can be calculated via a chain rule, with a closed-form gradient over eigenvalues, i.e., for $\mathbf{L}_{\text{norm}}=\text{Lap}(\mathbf{A}+\mathbf{C} \circ \mathbf{\Delta}^{\text{FM}})$, the derivatives of its $k$-th eigenvalue $\lambda_{k}$ is $\partial \lambda_{k}/\partial \mathbf{L}_{\text{norm}} = \mathbf{U}_{\cdot k}\mathbf{U}_{\cdot k}^{\top}$~\cite{rogers1970derivatives}.

\textbf{Scalability.}\label{scalability}
Due to the eigendecomposition, the time complexity of optimizing Eqs.(\ref{over-2},\ref{over-3}) is $\mathcal{O}(\text{Mn}^{3})$ with $\text{M}$ representing the number of iterations, due to the eigendecomposition, which is prohibitively expensive for large graphs.
To reduce the computational cost, instead of conducting eigendecomposition on the full graph spectrum, we employ selective eigendecomposition on the $K$ lowest eigenvalues via the Lanczos Algorithm~\cite{parlett1979lanczos}, which can reduce time complexity to $\mathcal{O}(\text{Mn}^{2}\text{K})$.
Similarly, we can use Truncated SVD~\cite{halko2011finding} to
obtain the $K$ highest eigenvalues of $\mathbf{X}$ and then concatenate them as approximate for eigendecomposition of $\widetilde{\mathbf{X}}$,  thereby reducing time complexity from $\mathcal{O}(\text{M(n+d)}^{3})$ to $\mathcal{O}(\text{M}\text{n} \log\text{K})$.
%

\section{Experiments}

%
In our general experimental settings, we use GIN~\cite{xu2018powerful} as the base encoder for all baselines to ensure a fair comparison. Each experiment is repeated 10 times with different random seeds, and we report the mean and standard derivation of the corresponding evaluation metrics. We select several best-performing baselines for comparison, including classic GCL methods, such as MVGRL~\cite{DBLP:conf/icml/HassaniA20}, InfoGraph~\cite{DBLP:conf/iclr/SunHV020}, GraphCL, and JOAO, as well as GCL methods with learnable graph augmentation, such as SEGA~\cite{DBLP:conf/icml/WuCSL023}, GCS~\cite{DBLP:conf/icml/WeiWBNBF23}, GCL-SPAN, AD-GCL, and AutoGCL. 

\subsection{Quantitative Evaluation}

\subsubsection{Comparison with State-of-the-Arts}

To comprehensively demonstrate the effectiveness and generalizability of CI-GCL, following previous studies~\cite{DBLP:conf/aaai/YinWHXZ22}, 
we perform evaluations for graph classification and regression under three different experimental settings: unsupervised, semi-supervised, and transfer learning.
\begin{table}[h]
\footnotesize
\centering
\caption{RMSE for unsupervised graph regression.}
\setlength{\tabcolsep}{1mm}{ 
\begin{tabular}{l|ccc|c}
\toprule
\rowcolor{gray!30}
{\textbf{Method}} &
{\textbf{molesol} $\downarrow$} & {\textbf{mollipo} $\downarrow$} & {\textbf{molfreesolv} $\downarrow$} &{\textbf{Avg.} $\downarrow$} \\
\midrule
\midrule
InfoGraph & 1.344$\pm$0.18 & 1.005$\pm$0.02 & 10.00$\pm$4.82 &4.118 \\
GraphCL   & 1.272$\pm$0.09 & 0.910$\pm$0.02 & 7.679$\pm$2.75 &3.287\\
MVGRL     & 1.433$\pm$0.15 & 0.962$\pm$0.04 & 9.024$\pm$1.98 &3.806 \\
JOAO      & 1.285$\pm$0.12 & 0.865$\pm$0.03 & 5.131$\pm$0.72 &2.427\\
GCL-SPAN  & 1.218$\pm$0.05 & \textbf{0.802$\pm$0.02} & \underline{4.531$\pm$0.46} 
&\underline{2.184} \\
AD-GCL    & \underline{1.217$\pm$0.09} & 0.842$\pm$0.03 & 5.150$\pm$0.62 &2.403\\
\midrule
\textbf{CI-GCL} & \textbf{1.130$\pm$0.13} & \underline{0.816$\pm$0.03} & \textbf{2.873$\pm$0.32} &\textbf{1.606} \\
\bottomrule
\end{tabular}}
\label{tab:unsup_rmse_auc-1}
\end{table}

\textbf{Unsupervised Learning.} 
We first train graph encoders (i.e., GIN) separately for each of the GCL baselines using unlabeled data. Then, we fix parameters of these models and train an SVM classifier using labeled data. We use TU datasets~\cite{Morris+2020} and OGB datasets~\cite{hu2020open} to evaluate graph classification and regression, respectively. We adopt the provided data split for the OGB datasets and use 10-fold cross-validation for the TU datasets as it lacks such a split. Table~\ref{tab:unsup_acc-1} shows the performance on graph classification and Table~\ref{tab:unsup_rmse_auc-1} draws the performance on graph regression.
In these tables,
CI-GCL achieves the best results on 9 datasets and competitive results on the MUTAG and mollipo datasets. 
Specifically, CI-GCL achieves the highest averaged accuracy in graph classification (77.74\%) and the lowest RMSE in graph regression (1.606), 
surpassing SOTA classification methods, such as SEGA (77.25\%), GraphCL (75.71\%), as well as SOTA regression methods, such as GCL-SPAN (2.184) and AD-GCL (2.403).

\begin{table*}[t]
\footnotesize
\centering
\caption{Accuracy (\%) for 10\% labeled semi-supervised graph classification.}
\setlength{\tabcolsep}{2.2mm}{ 
\begin{tabular}{l|ccccccc|c} 
\toprule
\rowcolor{gray!30}
\textbf{Method} &
{\textbf{NCI1} $\uparrow$} & {\textbf{PROTEINS} $\uparrow$} & {\textbf{DD} $\uparrow$}  & 
{\textbf{COLLAB} $\uparrow$} & {\textbf{RDT-B} $\uparrow$} & {\textbf{RDT-M5K} $\uparrow$} & {\textbf{GITHUB} $\uparrow$} & {\textbf{Avg.} $\uparrow$}\\
\midrule
\midrule
No Pre-train  & 73.72$\pm$0.2 & 70.40$\pm$1.5 & 73.56$\pm$0.4 & 73.71$\pm$0.3 & 86.63$\pm$0.3 & 51.33$\pm$0.4 & 60.87$\pm$0.2 &70.0\\
GraphCL   & 74.63$\pm$0.3 & 74.17$\pm$0.3 & 76.17$\pm$1.4 & 74.23$\pm$0.2 & 89.11$\pm$0.2 & 52.55$\pm$0.5 & 65.81$\pm$0.8 &72.3\\ 
JOAO      & 74.48$\pm$0.3 & 72.13$\pm$0.9 & 75.69$\pm$0.7 & 75.30$\pm$0.3 & 88.14$\pm$0.3 & 52.83$\pm$0.5 & 65.00$\pm$0.3 &71.9\\ 
SEGA      & 75.09$\pm$0.2  & 74.65$\pm$0.5  & 76.33$\pm$0.4       & 75.18$\pm$0.2    & 89.40$\pm$0.2  & \underline{53.73$\pm$0.3}       & \underline{66.01$\pm$0.7} & 72.9 \\
AD-GCL    & \underline{75.18$\pm$0.4} & 73.96$\pm$0.5 & \underline{77.91$\pm$0.7} & 75.82$\pm$0.3 & \underline{90.10$\pm$0.2} & 53.49$\pm$0.3 &         65.89$\pm$0.6    & \underline{73.1}   \\
AutoGCL   & 67.81$\pm$1.6 & \underline{75.03$\pm$3.5} & 77.50$\pm$4.4 & \textbf{77.16$\pm$1.5} & 79.20$\pm$3.5 & 49.91$\pm$2.7 & 58.91$\pm$1.5  &69.3    \\
\midrule
\textbf{CI-GCL}      & \textbf{75.86$\pm$0.8} & \textbf{76.28$\pm$0.3} & \textbf{78.01$\pm$0.9} & \underline{77.04$\pm$1.5} & \textbf{90.29$\pm$1.2} & \textbf{54.47$\pm$0.7} & \textbf{66.36$\pm$0.8}  &\textbf{74.0}\\
\bottomrule
\end{tabular}
}
\label{tab:semi_sup-1}
\end{table*}



\begin{table*}[t]
\caption{ROC-AUC (\%) for graph classification under transfer Learning settings.}
\footnotesize
\centering
\setlength{\tabcolsep}{1.46mm}{
\begin{tabular}{l|cccccccc|c|c} 
\toprule
\rowcolor{gray!30}
\textbf{Pre-Train} & 
\multicolumn{8}{c|}{\textbf{ZINC 2M}}  & \textbf{PPI-306K} & \\
\midrule
\rowcolor{gray!30}
\textbf{Fine-Tune} &
{\textbf{BBBP} $\uparrow$} & {\textbf{Tox21} $\uparrow$} & {\textbf{ToxCast} $\uparrow$}  & 
{\textbf{SIDER} $\uparrow$} & {\textbf{ClinTox} $\uparrow$} & {\textbf{MUV} $\uparrow$} & {\textbf{HIV} $\uparrow$} & {\textbf{BACE} $\uparrow$} & {\textbf{PPI} $\uparrow$}  & \textbf{Avg. $\uparrow$} \\
\midrule
\midrule
No Pre-train    & 65.8$\pm$4.5  & 74.0$\pm$0.8 & 63.4$\pm$0.6 & 57.3$\pm$1.6 & 58.0$\pm$4.4 & 71.8$\pm$2.5 & 75.3$\pm$1.9 & 70.1$\pm$5.4 & 64.8$\pm$1.0 & 66.7\\
MVGRL & 69.0$\pm$0.5  & 74.5$\pm$0.6 & 62.6$\pm$0.5 & 62.2$\pm$0.6 & 77.8$\pm$2.2 & 73.3$\pm$1.4 & 77.1$\pm$0.6 & 77.2$\pm$1.0 & 68.7$\pm$0.7 & 71.4 \\
SEGA & 71.9$\pm$1.1 & 76.7$\pm$0.4 & \underline{65.2$\pm$0.9} &  63.7$\pm$0.3 & \textbf{85.0$\pm$0.9} & \underline{76.6$\pm$2.5} & 77.6$\pm$1.4 & 77.1$\pm$0.5 & 68.7$\pm$0.5 &\underline{73.6} \\
GCS\textcolor{blue}{\ding{101}}   & \underline{72.5$\pm$0.5}  & 74.4$\pm$0.4 & 64.4$\pm$0.2 & 61.9$\pm$0.4 & 66.7$\pm$1.9 & \textbf{77.3$\pm$1.7} &  \underline{78.7$\pm$1.4} & \underline{82.3$\pm$0.3} & \underline{70.3$\pm$0.5} & 72.1 \\
GCL-SPAN        & 70.0$\pm$0.7  & \textbf{78.0$\pm$0.5} & 64.2$\pm$0.4 & \underline{64.7$\pm$0.5} & \underline{80.7$\pm$2.1} & 73.8$\pm$0.9 & 77.8$\pm$0.6 & 79.9$\pm$0.7 & 70.0$\pm$0.8 &73.2 \\
AD-GCL\textcolor{blue}{\ding{101}}           & 67.4$\pm$1.0  & 74.3$\pm$0.7 & 63.5$\pm$0.7 & 60.8$\pm$0.9 & 58.6$\pm$3.4 & 75.4$\pm$1.5 & 75.9$\pm$0.9 & 79.0$\pm$0.8 & 64.2$\pm$1.2  &68.7\\
AutoGCL\textcolor{blue}{\ding{101}}  & {72.0$\pm$0.6} & 75.5$\pm$0.3 & 63.4$\pm$0.4 & 62.5$\pm$0.6 & 79.9$\pm$3.3 &75.8$\pm$1.3 &77.4$\pm$0.6 &76.7$\pm$1.1 & 70.1$\pm$0.8 & 72.5 \\
\midrule
CI+AD-GCL & 68.4$\pm$1.1  & 74.5$\pm$0.9 & 64.0$\pm$0.8 & 61.4$\pm$0.9 & 59.8$\pm$3.2 & 76.5$\pm$1.7 & 77.0$\pm$0.9 & 80.0$\pm$0.8 & 65.3$\pm$1.1 & 69.6 \\
CI+AutoGCL & 73.9$\pm$0.7 & 76.4$\pm$0.3 & 63.8$\pm$0.3 & 63.9$\pm$0.6 & 80.9$\pm$3.1 &76.3$\pm$1.3 &78.8$\pm$0.7 &78.8$\pm$1.1 & 70.9$\pm$0.7 & 73.7\\
\midrule
\textbf{CI-GCL}            & \textbf{74.4$\pm$1.9}  & \underline{77.3$\pm$0.9} & \textbf{65.4$\pm$1.5} & \textbf{64.7$\pm$0.3} & 80.5$\pm$1.3 & 76.5$\pm$0.9 & \textbf{80.5$\pm$1.3} & \textbf{84.4$\pm$0.9} & \textbf{72.3$\pm$1.2} & \textbf{75.1} \\ 
\bottomrule
\end{tabular}
}
\label{tab:transfer-1}
\end{table*}

\textbf{Semi-Supervised Learning.} 
Following GraphCL, we employ 10-fold cross validation on each TU datasets using ResGCN~\cite{DBLP:conf/dsaa/PeiHIP21} as the classifier. For each fold, different from the unsupervised learning setting, we only use 10\% as labeled training data and 10\% as labeled testing data for graph classification. As shown in Table \ref{tab:semi_sup-1}, CI-GCL achieves highest averaged accuracy (74.0\%) compared with SOTA baselines SEGA (72.9\%) and AD-GCL (73.1\%).

\textbf{Transfer Learning.} 
To show the generalization, we conduct self-supervised pre-training for baselines on the pre-processed ZINC-2M or PPI-306K dataset~\cite{DBLP:conf/iclr/HuLGZLPL20} for 100 epochs and then fine-tune baselines on different downstream biochemical datasets. Table~\ref{tab:transfer-1} shows that CI-GCL  achieves best results on 6 datasets and comparable performance on the rest datasets with an averaged performance (75.1\%), comparing with SOTA baseline SEGA (73.6\%). 

\textbf{Summary.} From the above experimental results, we obtain the following three conclusions: (1) \underline{\textit{Higher Effectiveness.}} CI-GCL can achieve the best performance in three different experimental settings, attributed to its unified community-invariant constraint for graph augmentation. Compared to GraphCL and MVGRL, with similar contrastive objectives, the gain by CI-GCL mainly comes from the CI constraint and learnable graph augmentation procedure. While compared to AD-GCL and AutoGCL, with similar encoders, CI-GCL, guided by community invariance, is clearly more effective than the widely adopted uniformly random augmentation. (2) \underline{\textit{Better Generalizability.}} By maximizing spectral changes to minimize community changes, CI-GCL can obtain the encoder with better generalizability and transferability. Since the encoder is pre-trained to ignore the impact of community irrelevant information and mitigate the relationship between such information and downstream labels, solving the overfitting issue.
Furthermore, previous studies, such as JOAO and GCL-SPAN, improve generalizability of the GNN encoder on molecule classification by exploring structural information like subgraph.
We suggest that the community could be another way to study chemical and biological molecules.
(3) \underline{\textit{Wider Applicability.}} By combing the CI constraint with AD-GCL and AutoGCL in Table~\ref{tab:unsup_acc-1} and Table~\ref{tab:transfer-1}, we also see significant improvements on almost all datasets, showing that the CI constraint could be a plug-and-play component for any learnable GCL frameworks.

\begin{table}[h]
\footnotesize
\centering
\caption{Ablation study on unsupervised graph classification.
}
\setlength{\tabcolsep}{0.5mm}{ 
\begin{tabular}{l|cccc|c} 
\toprule
\rowcolor{gray!30}
\textbf{Method} &
{\scriptsize{\textbf{NCI1}} $\uparrow$}  & {\scriptsize{\textbf{PROTEIN}} $\uparrow$} & {\scriptsize{\textbf{\,\,DD\,\,}} $\uparrow$} & {\scriptsize{\textbf{MUTAG}} $\uparrow$}    &{\scriptsize{\textbf{Avg.}} $\uparrow$} \\
\midrule
w/o TA   & 78.8$\pm$0.6 & 74.6$\pm$0.6 & 77.8$\pm$0.8 & 86.1$\pm$0.9 & 79.3\\
w/o FA    & 79.2$\pm$0.4 & 75.1$\pm$0.2 & 78.1$\pm$0.3 & 86.3$\pm$0.9 &79.6 \\
\midrule
w/o CI on TA   & 79.9$\pm$0.6  & 75.7$\pm$0.9 & 78.8$\pm$0.3 & 87.6$\pm$0.5   & 80.5 \\ 
w/o CI on FA    & 80.0$\pm$0.3  & 75.9$\pm$1.4 & 78.7$\pm$0.8 & 87.5$\pm$1.4      & 80.5 \\
w/o CI on ALL     & 78.6$\pm$0.9  & 74.8$\pm$0.9 & 78.3$\pm$0.3 & 86.5$\pm$1.8  &79.5 \\ 
\midrule
CI-GCL     & \textbf{80.5$\pm$0.5}  & \textbf{76.5$\pm$0.1} & \textbf{79.6$\pm$0.3} & \textbf{89.6$\pm$0.9}  & \textbf{81.5}  \\
\bottomrule
\end{tabular}
}
\label{tab:ablation}
\end{table}
\subsubsection{Ablation Study}

We conduct an ablation study to evaluate the effectiveness of the proposed CI constraint on topology augmentation (TA) and feature augmentation (FA). We consider the following variants of CI-GCL: 
(1) \textbf{w/o TA}: remove TA on one branch.
(2) \textbf{w/o FA}: remove FA on one branch.
(3) \textbf{w/o CI on TA}: remove CI constraint on TA.
(4) \textbf{w/o CI on FA}: remove CI constraint on FA.
(5) \textbf{w/o CI on ALL}: remove CI constraint on both TA and FA.
Experimental results in Table~\ref{tab:ablation} demonstrate that the removal of either component of the method negatively impacts the ability of the graph representation learning to perform well. These results align with our hypothesis that random topology or feature augmentation without CI constraint corrupt community structure, thereby hindering model's performance in downstream tasks.

\subsection{Qualitative Evaluation}

\subsubsection{Robustness Against Various Noise}

To showcase the robustness of CI-GCL, we conduct experiments in the adversarial setting. Following GraphCL, we conduct \textit{Random} noise attack, with perturbation ratios $\sigma \in \{0.05,0.30\}$, on topology $\mathbf{A}$ and feature $\mathbf{X}$ of the input graph, respectively.
Specifically, for the topology attack, we randomly flip $\sigma \times m$ edges with $m$ as the total number of edges.
For the feature attack, we randomly select $\sigma \times d$ features to add Gaussian noise with $d$ as the total number of features.
Baselines without designed feature augmentation are set with random feature masking.
\begin{figure}[htb]
\centering
\includegraphics[scale=0.29, trim=10 20 0 0]{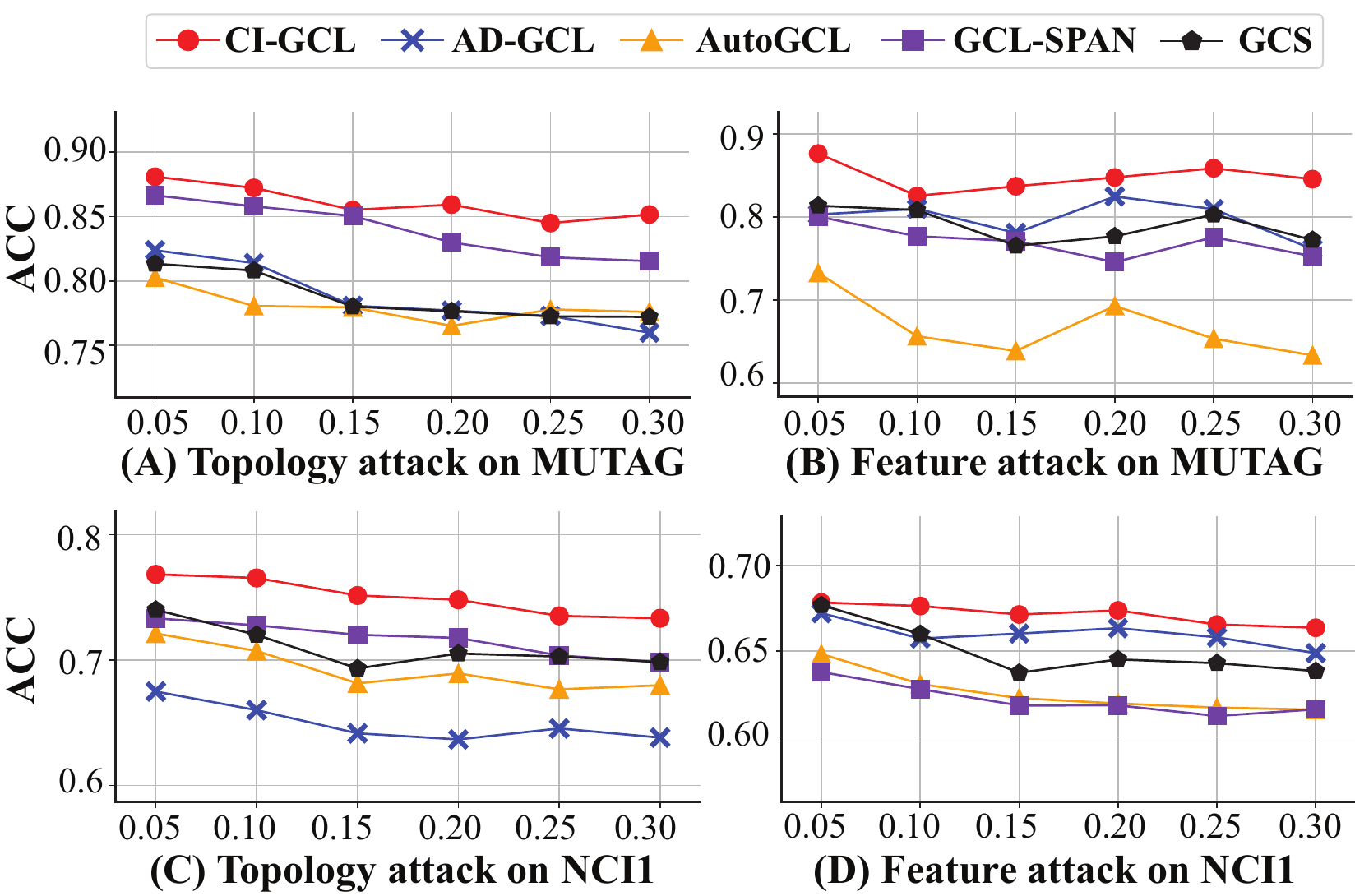}
\caption{Accuracy (\%) under noise attack on two datasets.}
\label{fig:noise_attack_mutag-1}
\end{figure}
Table~\ref{fig:noise_attack_mutag-1} reports the graph classification performance.
From Table~\ref{fig:noise_attack_mutag-1}, we have the following three findings.
(1) CI-GCL outperforms four best-performing GCL methods in both topology and feature attack, demonstrating its strong robustness.
(2)  CI-GCL and GCL-SPAN are more robustness than other baselines in the topology attack, showing that preserving high-level graph structure can improve robustness than random graph augmentation. While CI-GCL can better focus on community invariance to outperform GCL-SPAN. 
(3) CI-GCL is more robust than other baselines in the feature attack since we also apply the uniformed CI constraint to feature augmentation.
%

\begin{figure}[htb]
	\centering
	\includegraphics[scale=0.6, trim=0 20 0 0]{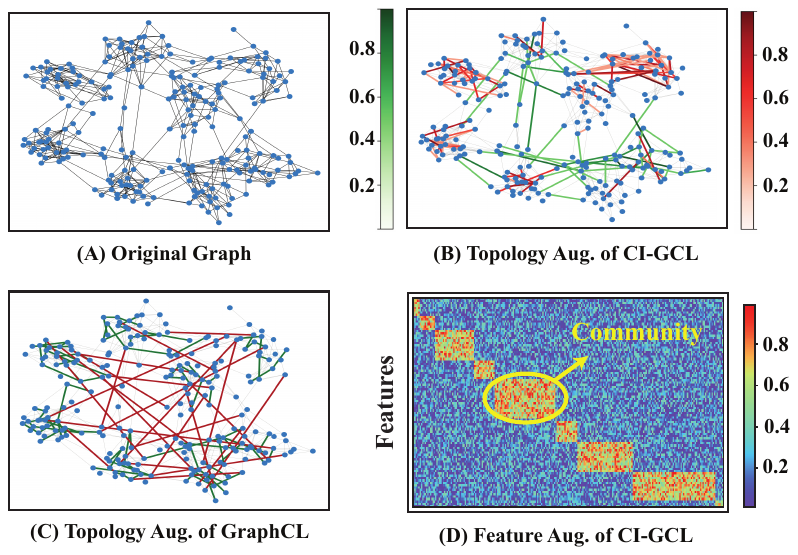}
	\caption{A case study of TA and FA of GraphCL and CI-GCL. (B-C) share the same color map and Green lines are edge dropping and Red lines are edge adding.}
    \label{fig:community}
\end{figure}

\subsubsection{Effectiveness in Community Preservation}

To explore the ability of community preservation, we draw community changes in Table~\ref{tab:community_change}, where community changes are defined as the averaged number of changes of community labels of nodes before and after graph augmentation by spectral clustering. 
We observe that CI-GCL can effectively preserve community structure due to the proposed CI constraint. Furthermore, refer to Table~\ref{tab:unsup_acc-1}, we can find that methods with larger community disruption, such as GraphCL and AutoGCL, underperform others with smaller community disruption. 
We also provide a visualization of CI-GCL on widely used synthetic graphs with 1,000 samples~\cite{DBLP:conf/iclr/KimO21}, that is suited for analysis since it possesses clear community structure~\cite{DBLP:conf/iclr/KimO21}.
We train our models in an unsupervised learning manner on this dataset and randomly select one example for visualization. In Figure~\ref{fig:community}(A-C), 
CI-GCL effectively preserve community structure by removing edges between clusters (Green lines) and add edges within each cluster (Red lines), 
while GraphCL destroys communities by randomly add and remove edges.
In Figure~\ref{fig:community}(D), with x- and y-axis represent nodes and features, respectively, CI-GCL can effectively maintain important features for community invariance.

\begin{table}[t]
\footnotesize
\centering
\caption{Community changes (\%) in unsupervised learning.}
\setlength{\tabcolsep}{1.2mm}{ 
\begin{tabular}{l|cccc|c}
\toprule
\rowcolor{gray!30}
{\textbf{Method}} &
{\scriptsize{\textbf{NCI1}} $\downarrow$} & {\scriptsize{\textbf{PROTEINS}} $\downarrow$} & {\scriptsize{\textbf{MUTAG}} $\downarrow$} &{\scriptsize{\textbf{IMDB-B}} $\downarrow$} &{\scriptsize{\textbf{Avg.}} $\downarrow$}\\
\midrule
 GraphCL    & 30.6 & 25.2 & 29.9 & 29.2 &28.7 \\
 GCS       & 27.6 & 30.9 & 26.5 & 32.5 &29.3\\
 GCL-SPAN  & 15.1 & 11.2 & 14.5 & 19.7 &15.1\\
 AD-GCL    & \underline{7.0} & \underline{4.8} & \underline{5.6} & \underline{18.0} &\underline{8.8}\\
 AutoGCL   & 34.2 & 31.0 & 30.3 & 33.4 &32.2\\
\midrule
 \textbf{CI-GCL} & \textbf{5.9} & \textbf{4.3} & \textbf{3.3} & \textbf{13.8} &\textbf{6.8}\\
\bottomrule
\end{tabular}}
\label{tab:community_change}
\end{table}

\section{Conclusion}

In this work, 
we aimed to propose an unified constraint that can be applied to both topology and feature augmentation, to ensure community invariance and benefit for downstream tasks.
To achieve this goal, we searched for the augmentation scheme that would maximize spectral changes of the input graph's topology and features, which can also minimize community changes.
Our proposed community-invariant constraint can be paired with various GCL frameworks.
We plan to explore more high-level graph information as constraints for learnable graph augmentation and apply our framework to many real-world applications in the future.

\section*{Impact Statements}
In this section, we elaborate on the broader impacts of our work from the following two aspects. 
(1) Learnable Graph Augmentations. With the rapid development of GCL, learnable graph augmentation has become a significant research topic in the machine-learning community. Compared to current learnable graph augmentation methods, our work introduces control over the augmentation scheme in joint learning settings. Considering that CI constraint as a kind of expert knowledge, we can perceive our work as establishing a connection between expert knowledge-guided augmentations and learnable augmentations through the design of specific constraints.
(2) Community and Spectrum: Despite significant advances in GCL, theoretical foundations regarding the relations between community preservation and spectrum design remain lacking. Our work highlights the significant potential of graph spectrum and community preservation in GCL, which may assist others in comprehending the graph spectrum. Moreover, we do not anticipate any direct negative impacts on society from our findings. 



\bibliography{main}
\bibliographystyle{icml2021}


\newpage
\appendix
\onecolumn

\section{Experiment Setup Details}

\setcounter{figure}{0}
\setcounter{table}{0}
\setcounter{equation}{0}
\renewcommand\theequation{A.\arabic{equation}}
\renewcommand\thefigure{A\arabic{figure}}
\renewcommand\thetable{A\arabic{table}}

\subsection{Hardware Specification and Environment}


We conduct our experiments using a single machine equipped with an Intel i9-10850K processor, Nvidia GeForce RTX 3090Ti (24GB) GPUs for the majority datasets. 
For the COLLAB, RDT-B, and RDT-M5K datasets, we utilized RTX A6000 GPUs (48GB) with batch sizes exceeding 512. 
The code is written in Python 3.10 and we use PyTorch 2.1.0 on CUDA 11.8 to train the model on the GPU. Implementation details can be accessed at the provided link
\begin{center}
\url{https://anonymous.4open.science/r/CI-GCL-E718}    
\end{center}

\subsection{Details on Datasets}

\subsubsection{TU datasets}
TU Datasets~\cite{Morris+2020}~\footnote{\href{https://chrsmrrs.github.io/datasets/}{https://chrsmrrs.github.io/datasets/}} provides a collection of benchmark datasets. We use several biochemical molecules and social networks for graph classification, as summarized in Table~\ref{tab:data_tudataset}. 
The data collection is also accessible in the PyG~\footnote{\href{https://pytorch-geometric.readthedocs.io/en/2.4.0/generated/torch_geometric.datasets.TUDataset.html}{https://pytorch-geometric.io}} library,
which employs a 10-fold evaluation data split. We used these datasets for the evaluation of the graph classification task in unsupervised and semi-supervised learning settings.

\begin{table}[h]
\caption{Statistics of TU Datasets for the unsupervised or semi-supervised graph classification task. \# means Number of.}
\centering
\begin{tabular}{l|c|c|c|c}
\toprule
\text { Dataset } & \text { Category } & \text { \# Graphs } & \text { Avg. \# Nodes } & \text { Avg. \# Edges } \\
\midrule
\midrule
\text { NCI1 } & \text { Biochemical Molecule } & 4,110 & 29.87 & 32.30 \\
\text { PROTEINS } & \text { Biochemical Molecule } & 1,113 & 39.06 & 72.82 \\
\text { DD } & \text { Biochemical Molecule } & 1,178 & 284.32 & 715.66 \\
\text { MUTAG } & \text { Biochemical Molecule } & 188 & 17.93 & 19.79 \\
\text { COLLAB } & \text { Social Network } & 5,000 & 74.49 & 2,457.78 \\
\text { RDT-B } & \text { Social Network } & 2,000 & 429.63 & 497.75\\
\text { RDT-M5K } & \text { Social Network } & 4,999 & 508.52 & 594.87 \\
\text { IMDB-B } & \text { Social Network } & 1,000 & 19.77 & 96.53 \\
\text { Github } & \text { Social Network } & 12,725 & 113.79 & 234.64 \\
\midrule
\end{tabular}
\label{tab:data_tudataset}
\end{table}

\begin{table}[h]
\caption{ Statistics of MoleculeNet dataset for the downstream transfer learning. \# means Number of.}
\centering
\begin{tabular}{l|c|c|c|c|c}
\toprule
\text { Dataset } & \text { Category } & \text { Utilization } & \text { \# Graphs } & \text { Avg. \# Nodes } & \text { Avg. \# Edges } \\
\midrule
\midrule
\text { ZINC-2M } & \text { Biochemical Molecule } & Pre-Traning & 2,000,000 & 26.62 & 57.72 \\
\text { PPI-306K } & \text { Biology Networks } & Pre-Traning & 306,925 & 39.82 & 729.62 \\
\midrule 
\text { BBBP } & \text { Biochemical Molecule } & FineTuning & 2,039 & 24.06 & 51.90 \\
\text { TOX21 } & \text { Biochemical Molecule } & FineTuning & 7,831 & 18.57 & 38.58 \\
\text { TOXCAST } & \text { Biochemical Molecule } & FineTuning & 8,576 & 18.78 & 38.52 \\
\text { SIDER } & \text { Biochemical Molecule } & FineTuning & 1,427 & 33.64 & 70.71 \\
\text { CLINTOX } & \text { Biochemical Molecule } & FineTuning & 1,477 & 26.15 & 55.76 \\
\text { MUV } & \text { Biochemical Molecule } & FineTuning & 93,087 & 24.23 & 52.55 \\
\text { HIV } & \text { Biochemical Molecule } & FineTuning & 41,127 & 25.51 & 54.93 \\
\text { BACE } & \text { Biochemical Molecule } & FineTuning & 1,513 & 34.08 & 73.71 \\
\text { PPI } & \text { Biology Networks } & FineTuning & 88,000 & 49.35 & 890.77 \\
\midrule
\end{tabular}
\label{tab:data_molecule_bio}
\end{table}

\subsubsection{MoleculeNet and PPI datasets}
We follow the same transfer learning setting of \citet{DBLP:conf/iclr/HuLGZLPL20}.
For pre-training, we use 2 million unlabeled molecules sampled from the ZINC15 database~\cite{DBLP:journals/jcisd/SterlingI15} for the chemistry domain, and 395K unlabeled protein ego-networks derived from PPI networks~\cite{mayr2018large} representing 50 species for the biology domain. 
During the fine-tuning stage, we use 8 larger binary classification datasets available in MoleculeNet~\cite{wu2018moleculenet} for the chemistry domain and PPI networks~\cite{zitnik2019evolution} for the biology domain. 
All these datasets are obtained from SNAP \footnote{\href{https://snap.stanford.edu/gnn-pretrain/}{https://snap.stanford.edu/gnn-pretrain/}}. 
Summaries of these datasets are presented in Table \ref{tab:data_molecule_bio}.

\subsubsection{OGB chemical molecular datasets}
Open Graph Benchmark (OGB) chemical molecular datasets~\cite{hu2020open} are employed for both graph classification and regression tasks in an unsupervised learning setting.
OGB~\footnote{\href{https://ogb.stanford.edu/}{https://ogb.stanford.edu/}} hosts datasets designed for chemical
molecular property classification and regression, 
as summarized in Table \ref{tab:data_ogb}. 
These datasets can be accessed through the OGB platform, and are also available within the PyG library~\footnote{\href{https://pytorch-geometric.readthedocs.io/en/2.4.0/modules/datasets.html}{https://pytorch-geometric.io}}.

\begin{table}[htb]
\caption{ Statistics of OGB chemical molecular dataset for both graph classification and regression tasks. }
\centering
\begin{tabular}{c|c|c|c|c|c}
\midrule 
\text { Dataset } & \text { Task Type } & \text { \# Graphs } & \text { Avg. \# Nodes } & \text { Avg. \# Edges } & \text { \# Task } \\
\midrule 
\text { ogbg-molesol } & \text { Regression } & 1,128 &  13.3 &  13.7 & 1 \\
\text { ogbg-molipo } & \text { Regression } &  4,200 &  27.0 & 29.5 & 1 \\
\text { ogbg-molfreesolv } & \text { Regression } &  642 & 8.7 & 8.4 & 1 \\
\midrule 
\text { ogbg-molbace } & \text { Classification } &  1,513 & 34.1 & 36.9 & 1 \\
\text { ogbg-molbbbp } & \text { Classification } & 2,039 & 24.1 & 26.0 & 1 \\
\text { ogbg-molclintox } & \text { Classification } & 1,477 & 26.2 & 27.9 & 2 \\
\text { ogbg-moltox21 } & \text { Classification } & 7,831 & 18.6 & 19.3 & 12 \\
\text { ogbg-molsider } & \text { Classification } &  1,427 & 33.6 & 35.4 & 27 \\
\midrule
\end{tabular}
\label{tab:data_ogb}
\end{table}

\subsection{Detailed Model Configurations, Training and Evaluation Process}

\subsubsection{Pre-Analysis experimental settings}

We outline detailed settings of the pre-analysis experiment for 
replication purposes. Specifically, our experiments are conducted on the IMDB-B dataset. Both the GraphCL and AD-GCL methods utilized a GIN encoder with identical architecture, including a graph-graph level contrastive loss, mean pooling readout, as well as hyperparameters, such as 2 convolutional layers with an embedding dimension of 32). Both strategies underwent 100 training iterations to acquire graph representations, which were subsequently evaluated by using them as features for a downstream SVM classifier.

\subsubsection{Unsupervised learning settings}

For unsupervised learning, we employ a 2-layer GIN~\cite{xu2018powerful} encoder with a 2-layer MLP for projection, followed by a mean polling readout function. The embedding size was set to 256 for both TU and OGB datasets. We conduct training for 100 epochs with a batch size of 256, utilizing the Adam optimizer with a learning rate of $0.01$. We adopt the provided data split for the OGB datasets and employ 10-fold cross-validation for the TU datasets, as the split is not provided. Initially, we pre-train the encoder using unlabeled data with a contrastive loss. Then, we fix the encoder and train with an SVM classifier for unsupervised graph classification or a linear prediction layer for unsupervised graph regression to evaluate the performance.  
Experiments are performed 10 times with mean and standard deviation of Accuracy(\%) scores for TUDataset classification, RMSE scores for OGB datasets regression, and ROC-AUC (\%) for OGB datasets classification.

\subsubsection{Semi-supervised learning settings}
Following GraphCL~\cite{DBLP:conf/nips/YouCSCWS20}, we employ a 10-fold cross-validation on each of the TU datasets using a ResGCN \cite{DBLP:conf/dsaa/PeiHIP21} classifier. For each fold, 10\% of each data is designed as labeled training data and 10\% as labeled testing data. These splits are randomly selected using the StratifiedKFold method~\footnote{\href{https://scikit-learn.org/stable/modules/generated/sklearn.model_selection.StratifiedKFold.html}{https://scikit-learn.org}}. 
During pre-training, we tune the learning rate in the range \{0.01, 0.001, 0.0001\} (using Adam optimizer) and the number of epochs in \{20, 40, 60, 80, 100\} through grid search. 
For fine-tuning, an additional linear graph prediction layer is added on top of the encoder to map the representations to the task labels. Experiments are performed 10 times with mean and standard deviation of Accuracy(\%) scores.

\subsubsection{Transfer learning settings}
In transfer learning, we first conduct self-supervised
pre-training on the pre-processed ZINC-2M or PPI-306K
dataset for 100 epochs. 
Subsequently, we fine-tune the backbone model (GIN encoder used in \cite{DBLP:conf/iclr/HuLGZLPL20}) on binary classification biochemical datasets. 
During the fine-tuning process, the encoder is equipped with an additional linear graph prediction layer on top, facilitating the mapping of representations to the task labels. Experiments are conducted 10 times, and the mean and standard deviation of ROC-AUC(\%) scores are reported.

\subsubsection{Adversarial learning settings}
In robustness experiments, we conduct random noise attacks by introducing edge perturbations to the graph topology, and introducing Gaussian noise to the attributes of the input graph.
Each experiment is performed within the context of unsupervised graph classification across various TU Datasets. Specifically, we first train the model using contrastive loss, following the same approach as in the unsupervised setting.  
During evaluation, we poison the topology and features of the input graph, and then asses its performance using an SVM classifier. We vary the ratio of added noise, denoted as $\sigma$, ranging from 0.05 to 0.3 with a step of 0.05. In the topological robustness experiment, we randomly add and remove $\sigma * m$ edges, where $m$ is the number of edges in a single graph. In the feature robustness experiment, we randomly select $\sigma * n * d$ positions to introduce noises, where $n$ is the number of nodes, and $d$ is the dimension of features.

\subsubsection{Synthetic dataset settings}

For the experiment depicted in Figure \ref{fig:community}, we randomly generate 1,000 synthetic graphs using Random Partition Graph method~\cite{DBLP:conf/iclr/KimO21}~\footnote{\href{https://pytorch-geometric.readthedocs.io/en/2.4.0/generated/torch_geometric.datasets.RandomPartitionGraphDataset.html}{https://pytorch-geometric.io}} and train models in an unsupervised learning manner. 
The parameters for graphs generation are set as follows: $num\_class$=8, $num\_nodes\_per\_class$=30, $node\_homophily\_ratio$=0.96, $average\_degree$=5.
After training, we select one graph to show the augmented edges,  where green lines represent edge dropping and red lines represent edge adding, as shown in Figure \ref{fig:community}(B). We compare it with the random edge perturbation in Figure \ref{fig:community}(C) GraphCL. Furthermore, Figure \ref{fig:community}(D) displays the inverted probability ($1-p$) for feature masking. The feature is obtained through the eigendecomposition of the corresponding adjacency matrix.

\section{More Discussion on Existing GCL Methods}

\setcounter{figure}{0}
\setcounter{table}{0}
\setcounter{equation}{0}
\renewcommand\theequation{B.\arabic{equation}}
\renewcommand\thefigure{B\arabic{figure}}
\renewcommand\thetable{B\arabic{table}}

\subsection{Introduction to Selected Best-performing Baselines}

Here, we briefly introduce some important baselines for graph self-supervised learning.

\begin{itemize}
    \item \textbf{MVGRL}~\cite{DBLP:conf/icml/HassaniA20} MVGRL maximizes the mutual information between the local Laplacian matrix and a global diffusion matrix.
    \item \textbf{InfoGraph}~\cite{DBLP:conf/iclr/SunHV020} InfoGraph maximizes the mutual information between the graph-level representation and the representations of substructures of different scales. 
    \item \textbf{Attribute Masking}~\cite{DBLP:conf/iclr/HuLGZLPL20} Attribute Masking learns the regularities of the node/edge attributes distributed over graph structure to capture inherent domain knowledge.
    \item \textbf{Context Prediction}~\cite{DBLP:conf/iclr/HuLGZLPL20} Context Prediction predicts the surrounding graph structures of subgraphs to pre-train a backbone GNN, so that it maps nodes appearing in similar structural contexts to nearby representations.
    \item \textbf{GraphCL}~\cite{DBLP:conf/nips/YouCSCWS20} GraphCL learns unsupervised representations of graph data through contrastive learning with random graph augmentations.
    \item \textbf{JOAO}~\cite{DBLP:conf/icml/YouCSW21} JOAO leverages GraphCL as the baseline model and automates the selection of augmentations when performing contrastive learning.
    \item \textbf{SEGA}~\cite{DBLP:conf/icml/WuCSL023} SEGA explores an anchor view that maintains the essential information of input graphs for graph contrastive learning, based on the theory of graph information bottleneck. Moreover, based on the structural information theory, we present a practical instantiation to implement this anchor view for graph contrastive learning.
    \item \textbf{AD-GCL}~\cite{suresh2021adversarial} AD-GCL aims to avoid capturing redundant information during the training by optimizing adversarial graph augmentation strategies in GCL and designing a trainable edge-dropping graph augmentation.
    \item \textbf{AutoGCL}~\cite{DBLP:conf/aaai/YinWHXZ22} AutoGCL employs a set of learnable graph view generators with node dropping and attribute masking and adopts a joint training strategy to train the learnable view generators, the graph encoder, and the classifier in an end-to-end manner. 
    We reproduce this algorithm by removing the feature expander part for the unsupervised setting.
    \item \textbf{GCS}~\cite{DBLP:conf/icml/WeiWBNBF23} GCS utilizes a gradient-based approach that employs contrastively trained models to preserve the semantic content of the graph. Subsequently, augmented views are generated by assigning different drop probabilities to nodes and edges.
    \item \textbf{GCL-SPAN}~\cite{lin2023spectral} GCL-SPAN develops spectral augmentation to guide topology augmentations. Two views are generated by maximizing and minimizing spectral changes, respectively. The probability matrices are pre-computed and serve as another input for the GCL framework. 
    We reproduce it by adding $10$-fold cross-validation and report the mean accuracy.
\end{itemize}

\subsection{Summary of the Most Similar Studies}

\begin{itemize}
    \item Compared with spectrum-based methods such as GCL-SPAN and GAME~\cite{DBLP:conf/nips/Liu0BSP22}. GCL-SPAN explores high-level graph structural information by disturbing it, i.e., they design their augmented view by corrupting what they want to retain. In contrast, CI-GCL preserves the core information between two augmented graphs to help the encoder better understand what should be retrained during graph augmentation.
    GAME focuses on the notion that spectrum changes of high-frequency should be larger than spectrum changes of low-frequency part,
    while CI-GCL considers all spectrum changes as a whole and endeavors to maintain community invariance during graph augmentation. 
    Another main difference between CI-GCL and these spectrum-based methods is that CI-GCL incorporates learnable graph augmentation, which makes it more flexible and generalizable for many different datasets. 
    \item Compared with community-based GCL methods such as gCooL~\cite{DBLP:conf/www/LiJT22},  ClusterSCL~\cite{DBLP:conf/www/WangZLDY0C22}, and CSGCL~\cite{DBLP:journals/corr/abs-2311-11073}.
    These methods first attempt to obtain community labels for each node using clustering methods such as K-means. Then they consider these community labels as the supervised signal to minimize the changes of community labels before and after the GCL framework.
    However, since they still employ random graph augmentation, they disrupt community structure during graph augmentation, which causes them to underperform some basic GCL baselines such as GraphCL on graph classification and node classification.
    Different from these methods, the proposed CI constraint can be applied for graph augmentation to improve effectiveness. 
    \item Compared with other GCL methods with learnable graph augmentation, such as AutoGCL and AD-GCL. 
    To the best of our knowledge, we propose the first unified community-invariant constraint for learnable graph augmentation. 
    With the proposed CI constraint, we can better maintain high-level graph information (i.e., community), which can enhance effectiveness, generalizability, and robustness for downstream tasks.
\end{itemize}

\section{Parameter Sensitive Experiments}

\setcounter{figure}{0}
\setcounter{table}{0}
\setcounter{equation}{0}
\renewcommand\theequation{C.\arabic{equation}}
\renewcommand\thefigure{C\arabic{figure}}
\renewcommand\thetable{C\arabic{table}}

\subsection{Hyperparameters Sensitivity in Overall Objective Function}

Before delving into the sensitivity of hyperparameters, we first introduce all the hyperparameters used in this study.
The overall objective function of CI-GCL in this study can be formulated as follows:
\begin{align}
    \min \  \mathcal{L}_{\text{overall}} &= \mathcal{L}_{\text{GCL}}-\boldsymbol{\alpha} \mathcal{L}_{\text{TA}}(\mathbf{\Delta}^{\text{TA}})- \boldsymbol{\beta} \mathcal{L}_{\text{FA}}(\mathbf{\Delta}^{\text{FA}}), \\
    \text{where}\quad \mathcal{L}_{\text{TA}}(\mathbf{\Delta}^{\text{TA}}) &= \mathcal{L}_{\text{ED}}(\mathbf{\Delta}^{\text{ED}}) + 
    \mathcal{L}_{\text{ND}}(\mathbf{\Delta}^{\text{ND}}) - 
    \mathcal{L}_{\text{EA}}(\mathbf{\Delta}^{\text{EA}}) , \\
    \text{and}\quad \mathcal{L}_{\text{FA}}(\mathbf{\Delta}^{\text{FA}}) &= \mathcal{L}_{\text{FM}}(\mathbf{\Delta}^{\text{FM}})
\end{align}
where $\text{TA}$ and $\text{FA}$ represent topology augmentation and feature augmentation, respectively. 
$\boldsymbol{\alpha}$ and $\boldsymbol{\beta}$ are hyperparameters controlling the strength of community-invariant constraints on topology and feature augmentation, respectively. 
Another important hyperparameter $K$, as mentioned in Section~\ref{scalability}, denotes the number of eigenvalues and eigenvectors.

We evaluate parameter sensitivity by varying the target hyperparameter while keeping other hyperparameters fixed.
We select three datasets MUTAG, IMDB-B, and PROTEINS, from the TU datasets to evaluate parameter sensitivity in the unsupervised setting using 10-fold cross-validation.
In our experimental settings, we adopt a 2-layer GIN encoder with a 2-layer MLP for projection, followed by a mean pooling readout function. The embedding size is fixed at 256. 
We conduct training for 100 epochs with a batch size of 256. 
Throughout all experiments, we employ the Adam optimizer with a learning rate of $0.01$.

\paragraph{Effect of $K$} 
Fixing $\boldsymbol{\alpha}=0.8$ and $\boldsymbol{\beta}=1.0$ unchanged, 
we vary $K$ from 2 to 20 with a step of 1 to evaluate its performance. As shown in Figure~\ref{fig:parameter_sens} (the first column), CI-GCL demonstrates significant improvement as the selected number increases from $2$ to $6$. We attribute this improvement to the small number of spectrum that only contain a rough community structure of the graph. Conducting topology or feature augmentations constrained by such an incomplete community structure may further compromise the distinguishability of the generated positive views. 
Therefore, it is important to use a sufficient number of spectrum to obtain a confident augmentation view. Figure \ref{fig:parameter_sens} demonstrates that increasing the number of spectrum substantially enhances performance, with the best results achieved at $6$. 
While MUTAG, IMDB-B, and PROTEINS are relatively small datasets, we recommend using a larger number of spectrum for larger graphs.

\begin{figure}[h]
    \centering
    \includegraphics[scale=0.6]{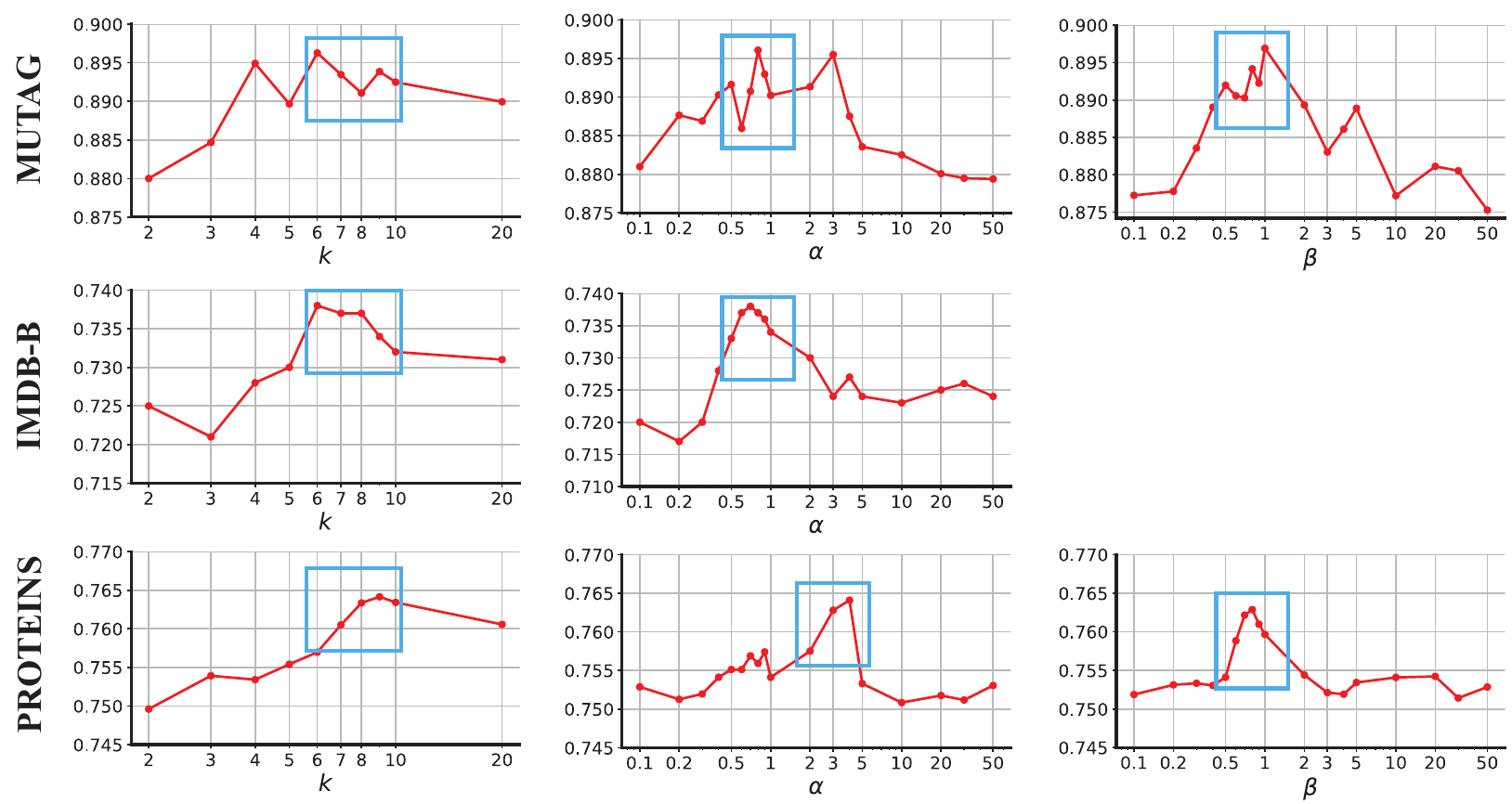}
    \caption{The parameter sensitivity involves the selected number of spectrum $K$, the balance weight of the topological constraint $\boldsymbol{\alpha}$, and the balance weight of the feature constraint $\boldsymbol{\beta}$.}
    \label{fig:parameter_sens}
\end{figure}

\paragraph{Effect of $\boldsymbol{\alpha}$}
Fixing $\boldsymbol{\beta}=1.0$ and $K=6$ unchanged,
We vary $\boldsymbol{\alpha}$ from $0.1$ to $50$ with a step of 0.1 (from $0.1$ to $1$), and 1 (from $1$ to $5$), and 10 (from $10$ to $30$) to evaluate its performance. As shown in Figure~\ref{fig:parameter_sens} (the second column),
CI-GCL demonstrates significant improvement when the weight of the topological constraint increases from $0.1$ to $3$. We attribute this improvement to the weight of $\mathcal{L}_{\text{TA}}$ controlling the rate of community invariance. However, if we further increase the value, we observe a decreasing trend in performance. This may occur because the topological community constraint may dominate the overall loss, leading to overfitting of the topological augmentation view. We recommend setting the best parameter for $\alpha$ within the range of ($0.7 \sim 1.0$). For datasets with an average number of nodes larger than 40, we suggest setting $\alpha$ within the range of ($2 \sim 5$).

\paragraph{Effect of $\boldsymbol{\beta}$}
Fixing $\boldsymbol{\alpha}=0.8$ and $K=6$ unchanged,
we vary $\beta$ from $0.1$ to $50$ using the same step range as $\boldsymbol{\alpha}$. 
As shown in Figure~\ref{fig:parameter_sens} (the third column), 
CI-GCL demonstrates significant improvement when the weight of feature-wise constraint increases from $0.1$ to $1.0$. 
We attribute this improvement to the weight of $\mathcal{L}_{FM}$, which controls the selection of important features. 
However, if we further increase the value, we observe a decreasing trend in performance. This may also occur because the feature community constraint may dominate the overall loss, leading to overfitting of the feature augmentation view. We recommend setting the best parameter for $\beta$ within the range of ($0.5 \sim 1.0$).

\subsection{Analysis of Perturbation Strength}

The value of $\psi$ controls the strength of perturbation when generating augmented graphs. 
A larger $\psi$ value indicates that more edges will be either dropped or added. The number of perturbed edges in each graph is restricted to $\psi = \sigma_e \times m$, where $m$ represents the number of edges in the input graph and $\sigma_e$ denotes the perturbation rate. To analyze the impact of perturbation strength, we examine the influence of the perturbation ratio $\sigma_e$ on augmentations. We vary $\sigma_e$ from 0.05 to 0.5 with a step size of 0.05 across datasets NCI1, PROTEINS, DD, and MUTAG. The performance comparison is illustrated in Figure \ref{fig:pert_ratio_change}, conducted within the context of an unsupervised graph classification task.
We recommend selecting the optimal perturbation strength within the range of $0.1$ to $0.3$.

\begin{figure}[htb]
    \centering
    \includegraphics[scale=0.605]{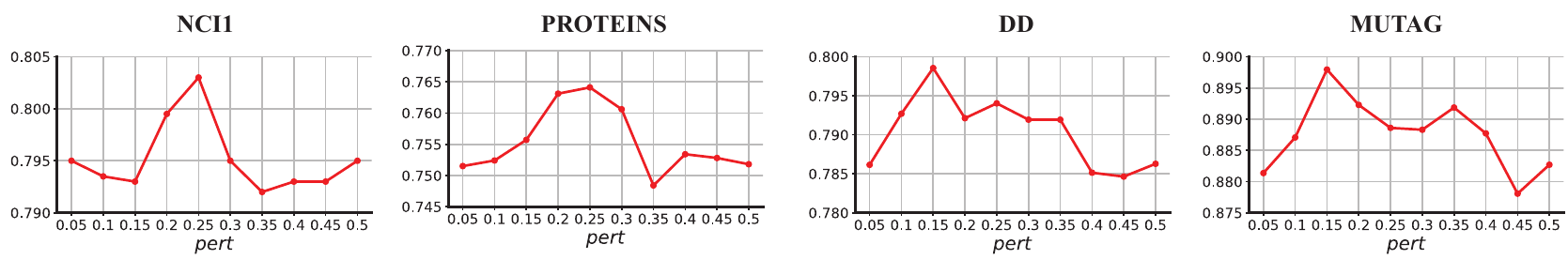}
    \caption{Graph classification performance when tuning perturbation ratio $\sigma_e$.}
    \label{fig:pert_ratio_change}
\end{figure}

\section{More Experimental Results}

\setcounter{figure}{0}
\setcounter{table}{0}
\setcounter{equation}{0}
\renewcommand\theequation{D.\arabic{equation}}
\renewcommand\thefigure{D\arabic{figure}}
\renewcommand\thetable{D\arabic{table}}

\begin{table*}[htb]
\caption{Unsupervised graph classification. We pre-train using the whole dataset to learn graph embeddings and feed them into a downstream SVM classifier with 10-fold cross-validation.}
\setlength{\tabcolsep}{0.52mm}{ 
\begin{tabular}{l|cccccccc|c}
\toprule
\rowcolor{gray!30}
\textbf{Method} &
{\textbf{NCI1}} & {\textbf{PROTEINS}} & {\textbf{DD}} & {\textbf{MUTAG}} & 
{\textbf{COLLAB}} & {\textbf{RDT-B}} & {\textbf{RDT-M5K}} & {\textbf{IMDB-B}} & {\textbf{Avg.}}\\
\specialrule{\arrayrulewidth}{0pt}{\belowrulesep}
\specialrule{\arrayrulewidth}{\abovetopsep}{0pt}
GL   & -  & -                & - & 81.66$\pm$2.1 & - & 77.34$\pm$0.1 & 41.01$\pm$0.1 & 65.87$\pm$0.9 & - \\
WL   & 80.01$\pm$0.5 & 72.92$\pm$0.5 & - & 80.72$\pm$3.0 & - & 68.82$\pm$0.4 & 46.06$\pm$0.2 & 72.30$\pm$3.4 & -  \\
DGK  & 80.31$\pm$0.4 & 73.30$\pm$0.8 & - & 87.44$\pm$2.7 & - & 78.04$\pm$0.4 & 41.27$\pm$0.2 & 66.96$\pm$0.5 & - \\ 
\midrule
node2vec  & 54.89$\pm$1.6 & 57.49$\pm$3.5 & - & 72.63$\pm$10.2 & - & - & - & -  & - \\
sub2vec   & 52.84$\pm$1.5 & 53.03$\pm$5.5 & - & 61.05$\pm$15.8 & - & 71.48$\pm$0.4 & 36.68$\pm$0.42 & 55.26$\pm$1.5 & -  \\
graph2vec & 73.22$\pm$1.8 & 73.30$\pm$2.0 & - & 83.15$\pm$9.2 & - & 75.78$\pm$1.0 & 47.86$\pm$0.26 & 71.10$\pm$0.5 & - \\
\midrule
InfoGraph & 76.20$\pm$1.0 & 74.44$\pm$0.3 & 72.85$\pm$1.7 & 89.01$\pm$1.1 & 70.65$\pm$1.1 & 82.50$\pm$1.4 & 53.46$\pm$1.0 & 73.03$\pm$0.8 & 74.02 \\
GraphCL   & 77.87$\pm$0.4 & 74.39$\pm$0.4 & 78.62$\pm$0.4 & 86.80$\pm$1.3 & 71.36$\pm$1.1 & 89.53$\pm$0.8 & 55.99$\pm$0.3 & 71.14$\pm$0.4 & 75.71 \\
MVGRL     & 76.64$\pm$0.3 & 74.02$\pm$0.3        & 75.20$\pm$0.4 & 75.40$\pm$7.8 & 73.10$\pm$0.6 & 82.00$\pm$1.1 & 51.87$\pm$0.6 & 63.60$\pm$4.2 & 71.48 \\
JOAO      & 78.07$\pm$0.4 & 74.55$\pm$0.4 & 77.32$\pm$0.5 & 87.35$\pm$1.0 & 69.50$\pm$0.3 & 85.29$\pm$1.4 & 55.74$\pm$0.6 &  70.21$\pm$3.0  & 74.75\\
SEGA      &  \underline{79.00$\pm$0.7} & \underline{76.01$\pm$0.4} & 78.76$\pm$0.6 & \textbf{90.21$\pm$0.7} & \underline{74.12$\pm$0.5} &  \underline{90.21$\pm$0.7} & 56.13$\pm$0.3 & \underline{73.58$\pm$0.4} & \underline{77.25} \\
GCS\textcolor{blue}{\ding{101}}       & 77.18$\pm$0.3 & 74.04$\pm$0.4 & 76.28$\pm$0.3 & 88.19$\pm$0.9 & 74.00$\pm$0.4 & 86.50$\pm$0.3 & \underline{56.30$\pm$0.3} & 72.90$\pm$0.5 & 75.64 \\
GCL-SPAN\textcolor{blue}{\ding{101}}  & 75.43$\pm$0.4 & 75.78$\pm$0.4 & \underline{78.78$\pm$0.5} & 85.00$\pm$0.8 & 71.40$\pm$0.5 & 86.50$\pm$0.1 & 54.10$\pm$0.5 & 66.00$\pm$0.7 & 74.12 \\
AD-GCL\textcolor{blue}{\ding{101}}    & 73.38$\pm$0.5 & 73.59$\pm$0.7 & 75.10$\pm$0.4 & 89.70$\pm$1.0 & 72.50$\pm$0.6 & 85.52$\pm$0.8 & 54.91$\pm$0.4 & 71.50$\pm$0.6 & 74.53 \\
AutoGCL\textcolor{blue}{\ding{101}}   & 78.32$\pm$0.5 & 69.73$\pm$0.4 & 75.75$\pm$0.6 & 85.15$\pm$1.1 & 71.40$\pm$0.7 & 86.60$\pm$1.5 & 55.71$\pm$0.2 & 72.00$\pm$0.4 & 74.33 \\
\midrule
\textsc{w/o CI on ALL } & 78.67$\pm$0.9 & 74.89$\pm$0.9 & 78.33$\pm$0.3 & 86.53$\pm$1.8 & -    & -        & -         & - & - \\
\textsc{w/o CI on TA} & 79.98$\pm$0.6 & 75.74$\pm$0.9 & 78.84$\pm$0.3 & 87.65$\pm$0.5 & 73.49$\pm$0.9 & 88.93$\pm$0.5 & 55.36$\pm$0.3 & 70.34$\pm$0.8 & 76.29 \\
\textsc{w/o CI on FA} & 80.02$\pm$0.3 & 75.99$\pm$1.4 & 78.75$\pm$0.8 & 87.56$\pm$1.4 & -      & -         & -         & - & - \\
\textsc{CI-GCL} & \textbf{80.50$\pm$0.5} & \textbf{76.50$\pm$0.1} & \textbf{79.63$\pm$0.3} & 89.67$\pm$0.9 & \textbf{74.40$\pm$0.6} & \textbf{90.80$\pm$0.5} & \textbf{56.57$\pm$0.3} & \textbf{73.85$\pm$0.8} & \textbf{77.74}\\

\specialrule{\arrayrulewidth}{0pt}{\belowbottomsep}
\end{tabular}}
\label{tab:full_results_unsup_acc}
\end{table*}

\subsection{Complete Results for Unsupervised Learning}

\textbf{Baselines.} We compare CI-GCL with kernel-based methods, such as Graphlet Kernel (GL)~\cite{DBLP:journals/jmlr/ShervashidzeVPMB09}, Weisfeiler-Lehman sub-tree kernel (WL)~\cite{DBLP:journals/jmlr/ShervashidzeSLMB11}, and Deep Graph Kernel (DGK)~\cite{DBLP:conf/kdd/YanardagV15}. 
Unsupervised graph learning methods like node2vec~\cite{DBLP:conf/kdd/GroverL16}, sub2vec~\cite{DBLP:conf/pakdd/AdhikariZRP18} and graph2vec~\cite{DBLP:journals/corr/NarayananCVCLJ17}. 
Classic GCL methods, such as  MVGRL~\cite{DBLP:conf/icml/HassaniA20}, InfoGraph~\cite{DBLP:conf/iclr/SunHV020}, GraphCL~\cite{DBLP:conf/nips/YouCSCWS20}, and JOAO~\cite{DBLP:conf/icml/YouCSW21},
SEGA~\cite{DBLP:conf/icml/WuCSL023}.
And learnable GCL methods, such as   
GCS~\cite{DBLP:conf/icml/WeiWBNBF23},
GCL-SPAN~\cite{lin2023spectral}, 
AD-GCL~\cite{suresh2021adversarial}, AutoGCL~\cite{DBLP:conf/aaai/YinWHXZ22}.
The complete results for unsupervised graph classification results are shown in Table \ref{tab:full_results_unsup_acc}.
We also report the complete results on OGB dataset in Table \ref{tab:full_results_unsup_rmse_auc}. 

\begin{table*}[htbp]
\caption{Unsupervised graph regression (measured by RMSE) and classification (measured by ROC-AUC) results on OGB datasets. We pre-train using the whole dataset to learn graph embeddings and feed them into a downstream classifier with scaffold split. Each experiment is performed 10 times and take the average accuracy as a result.
}
\setlength{\tabcolsep}{1.0mm}{ 
\begin{tabular}{l|ccc|ccccc}
\specialrule{0.1em}{\abovetopsep}{0pt}
\rowcolor{gray!30}

{\textbf{Task}}& 
\multicolumn{3}{c|}{\textbf{Regression (RMSE)}}  & \multicolumn{5}{c}{\textbf{Classification (AUC)}} \\
\specialrule{0.1em}{\abovetopsep}{0pt}
\rowcolor{gray!30} 

{\textbf{Method}} &
{\textbf{molesol}} & {\textbf{mollipo}} & {\textbf{molfreesolv}} & {\textbf{molbace}} & 
{\textbf{molbbbp}} & {\textbf{molclintox}} & {\textbf{moltox21}} & {\textbf{molsider}} \\
\specialrule{\arrayrulewidth}{0pt}{\belowrulesep}
\specialrule{\arrayrulewidth}{\abovetopsep}{0pt}
InfoGraph & 1.344$\pm$ 0.18 & 1.005$\pm$0.02 & 10.005$\pm$4.82 & 74.74$\pm$3.60 & 66.33$\pm$2.79 & 64.50$\pm$5.32 & 69.74$\pm$0.57 & 60.54$\pm$0.90 \\
GraphCL   & 1.272$\pm$0.09 & 0.910$\pm$0.02 & 7.679$\pm$2.75 & 74.32$\pm$2.70 & 68.22$\pm$1.89 & 74.92$\pm$4.42 & 72.40$\pm$1.01 & 61.76$\pm$1.11 \\
MVGRL     & 1.433$\pm$0.15 & 0.962$\pm$0.04 & 9.024$\pm$1.98 & 74.20$\pm$2.31 & 67.24$\pm$1.39 & 73.84$\pm$4.25 & 70.48$\pm$0.83 & 61.94$\pm$0.94 \\
JOAO      & 1.285$\pm$0.12 & 0.865$\pm$0.03 & 5.131$\pm$0.72 & 74.43$\pm$1.94 & 67.62$\pm$1.29 & 78.21$\pm$4.12 & 71.83$\pm$0.92 & 62.73$\pm$0.92 \\
GCL-SPAN  & 1.218$\pm$0.05 & \textbf{0.802$\pm$0.02} & \underline{4.531$\pm$0.46} & \underline{76.74$\pm$2.02} & \underline{69.59$\pm$1.34} & \underline{80.28$\pm$2.42} & \underline{72.83$\pm$0.62} & \underline{64.87$\pm$0.88}  \\
AD-GCL    & \underline{1.217$\pm$0.09} & 0.842$\pm$0.03 & 5.150$\pm$0.62 & 76.37$\pm$2.03 & 68.24$\pm$1.47 & \textbf{80.77$\pm$3.92} & 71.42$\pm$0.73 & 63.19$\pm$0.95 \\
\midrule
\textsc{CI-GCL} & \textbf{1.130$\pm$0.13} & \underline{0.816$\pm$0.03} & \textbf{2.873$\pm$0.32} & \textbf{77.26$\pm$1.51} & \textbf{70.70$\pm$0.67} & 78.90$\pm$2.59 & \textbf{73.59$\pm$0.66} & \textbf{65.91$\pm$0.82} \\
\specialrule{\arrayrulewidth}{0pt}{\belowbottomsep}
\end{tabular}}
\label{tab:full_results_unsup_rmse_auc}
\end{table*}

\subsection{Transfer learning setting}

The full results for transfer learning are shown in Table \ref{tab:full_results_transfer}.

\begin{table*}[htb]
\centering
\caption{Transfer Learning is conducted with various manually designed pre-training schemes, 
with different manually designed pre-training schemes, following the pre-training setting and dataset splitting as outlined in \cite{DBLP:conf/iclr/HuLGZLPL20}. Each experiment is repeated 10 times, with mean and standard deviation calculated for the ROC-AUC scores.
}
\renewcommand{\arraystretch}{1.02} 
\setlength{\tabcolsep}{1.1mm}{ 
\begin{tabular}{l|cccccccc|c|c} 
\specialrule{0.1em}{\abovetopsep}{0pt}
\rowcolor{gray!30}
\textbf{Pre-Train} & 
\multicolumn{8}{c|}{\textbf{ZINC 2M}}  & \textbf{PPI-306K} & \\
\specialrule{0.1em}{\abovetopsep}{0pt}
\rowcolor{gray!30}

\textbf{Fine-Tune} &
{\textbf{BBBP}} & {\textbf{Tox21}} & {\textbf{ToxCast}}  & 
{\textbf{SIDER}} & {\textbf{ClinTox}} & {\textbf{MUV}} & {\textbf{HIV}} & {\textbf{BACE}} & {\textbf{PPI}} & \textbf{Avg.} \\

\specialrule{\arrayrulewidth}{0pt}{\belowrulesep}
\specialrule{\arrayrulewidth}{\abovetopsep}{0pt}

No Pre-train    & 65.8$\pm$4.5  & 74.0$\pm$0.8 & 63.4$\pm$0.6 & 57.3$\pm$1.6 & 58.0$\pm$4.4 & 71.8$\pm$2.5 & 75.3$\pm$1.9 & 70.1$\pm$5.4 & 64.8$\pm$1.0 & 66.72 \\
Infomax         & 68.8$\pm$0.8  & 75.3$\pm$0.5 & 62.7$\pm$0.4 & 58.4$\pm$0.8 & 69.9$\pm$3.0 & 75.3$\pm$2.5 & 76.0$\pm$0.7 & 75.9$\pm$1.6 & 64.1$\pm$1.5 & 69.60 \\
EdgePred        & 67.3$\pm$2.4  & 76.0$\pm$0.6 & 64.1$\pm$0.6 & 60.4$\pm$0.7 & 64.1$\pm$3.7 & 74.1$\pm$2.1 &  76.3$\pm$1.0 & 79.9$\pm$0.9 & 65.7$\pm$1.3 & 69.76 \\ 
AttrMasking     & 64.3$\pm$2.8  & 76.7$\pm$0.4 & 64.2$\pm$0.5 & 61.0$\pm$0.7 & 71.8$\pm$4.1 & 74.7$\pm$1.4 & 77.2$\pm$1.1 & 79.3$\pm$1.6 & 65.2$\pm$1.6 & 70.48\\ 
ContextPred     & 68.0$\pm$2.0  & 75.7$\pm$0.7 & 63.9$\pm$0.6 & 60.9$\pm$0.6 & 65.9$\pm$3.8 & 75.8$\pm$1.7 & 77.3$\pm$1.0 & 79.6$\pm$1.2 & 64.4$\pm$1.3 & 70.16\\ 
GraphCL         & 69.7$\pm$0.7  & 73.9$\pm$0.7 & 62.4$\pm$0.6 & 60.5$\pm$0.9 & 76.0$\pm$2.7 & 69.8$\pm$2.7 & 78.5$\pm$1.2 & 75.4$\pm$1.4 & 67.9$\pm$0.9 & 70.45\\ 
MVGRL & 69.0$\pm$0.5  & 74.5$\pm$0.6 & 62.6$\pm$0.5 & 62.2$\pm$0.6 & 77.8$\pm$2.2 & 73.3$\pm$1.4 & 77.1$\pm$0.6 & 77.2$\pm$1.0 & 68.7$\pm$0.7 & 71.37 \\
JOAO            & 70.2$\pm$1.0  & 75.0$\pm$0.3 & 62.9$\pm$0.5 & 60.0$\pm$0.8 & 81.3$\pm$2.5 & 71.7$\pm$1.4 & 76.7$\pm$1.2 & 77.3$\pm$0.5 & 64.4$\pm$1.4 & 71.05\\ 
SEGA & 71.9$\pm$1.1 & 76.7$\pm$0.4 & \underline{65.2$\pm$0.9} &  63.7$\pm$0.3 & \textbf{85.0$\pm$0.9} & \underline{76.6$\pm$2.5} & 77.6$\pm$1.4 & 77.1$\pm$0.5 & 68.7$\pm$0.5 &\underline{73.61} \\
GCS\textcolor{blue}{\ding{101}}   & \underline{72.5$\pm$0.5}  & 74.4$\pm$0.4 & 64.4$\pm$0.2 & 61.9$\pm$0.4 & 66.7$\pm$1.9 & \textbf{77.3$\pm$1.7} &  \underline{78.7$\pm$1.4} & \underline{82.3$\pm$0.3} & \underline{70.3$\pm$0.5} & 72.05 \\
GCL-SPAN        & 70.0$\pm$0.7  & \textbf{78.0$\pm$0.5} & 64.2$\pm$0.4 & \underline{64.7$\pm$0.5} & \underline{80.7$\pm$2.1} & 73.8$\pm$0.9 & 77.8$\pm$0.6 & 79.9$\pm$0.7 & 70.0$\pm$0.8 &73.23 \\
AD-GCL\textcolor{blue}{\ding{101}}           & 67.4$\pm$1.0  & 74.3$\pm$0.7 & 63.5$\pm$0.7 & 60.8$\pm$0.9 & 58.6$\pm$3.4 & 75.4$\pm$1.5 & 75.9$\pm$0.9 & 79.0$\pm$0.8 & 64.2$\pm$1.2  &68.79\\
AutoGCL\textcolor{blue}{\ding{101}}  & {72.0$\pm$0.6} & 75.5$\pm$0.3 & 63.4$\pm$0.4 & 62.5$\pm$0.6 & 79.9$\pm$3.3 &75.8$\pm$1.3 &77.4$\pm$0.6 &76.7$\pm$1.1 & 70.1$\pm$0.8 & 72.59 \\
\midrule
\textbf{CI-GCL}            & \textbf{74.4$\pm$1.9}  & \underline{77.3$\pm$0.9} & \textbf{65.4$\pm$1.5} & \textbf{64.7$\pm$0.3} & 80.5$\pm$1.3 & 76.5$\pm$0.9 & \textbf{80.5$\pm$1.3} & \textbf{84.4$\pm$0.9} & \textbf{72.3$\pm$1.2} & \textbf{75.11} \\ 
\specialrule{\arrayrulewidth}{0pt}{\belowbottomsep}
\end{tabular}
}
\label{tab:full_results_transfer}
\end{table*}

\newpage 

\subsection{Adversarial Attacks Experiment}

The full results are shown in Figure \ref{fig:noise_attack_topo} and Figure \ref{fig:noise_attack_fea}. As the results show, CI-GCL exhibits greater robustness than other baselines against both topology-wise and feature-wise attacks.

\begin{figure}[htb]
	\centering
	\includegraphics[scale=0.6, trim=10 10 10 0]{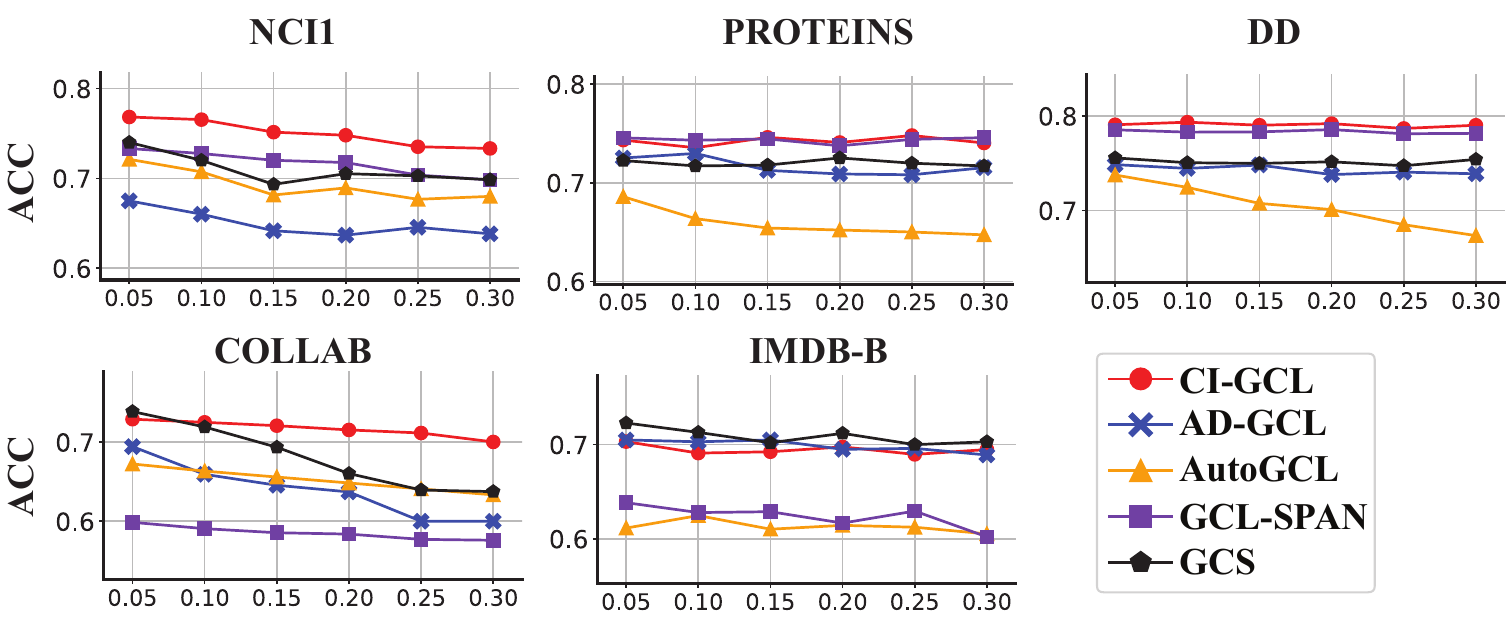}
    \caption{Accuracy under noise attack on graph's topology.}
    \label{fig:noise_attack_topo}
\end{figure}

\begin{figure}[htb]
	\centering
	\includegraphics[scale=0.6, trim=20 0 20 0]{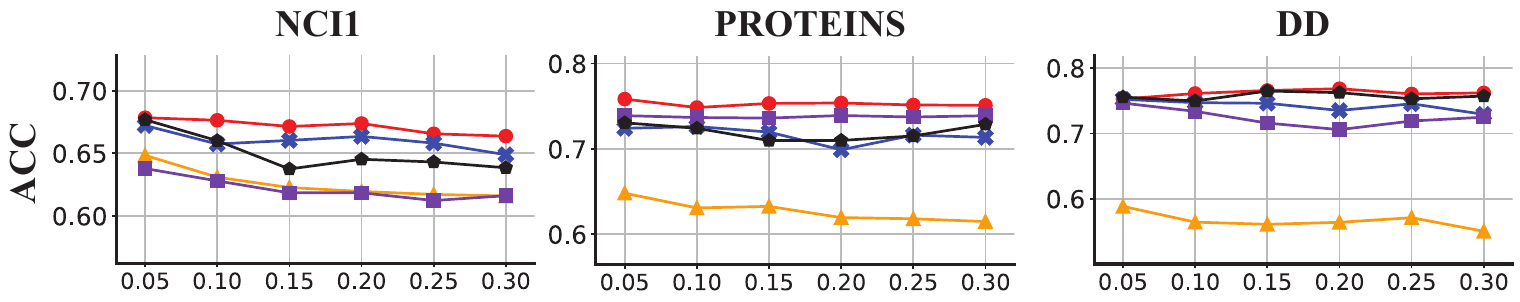}
    \caption{Accuracy under noise attack on graph's attributes.}
    \label{fig:noise_attack_fea}
\end{figure}

\begin{figure}[htb]
    \centering
    \includegraphics[scale=0.58]{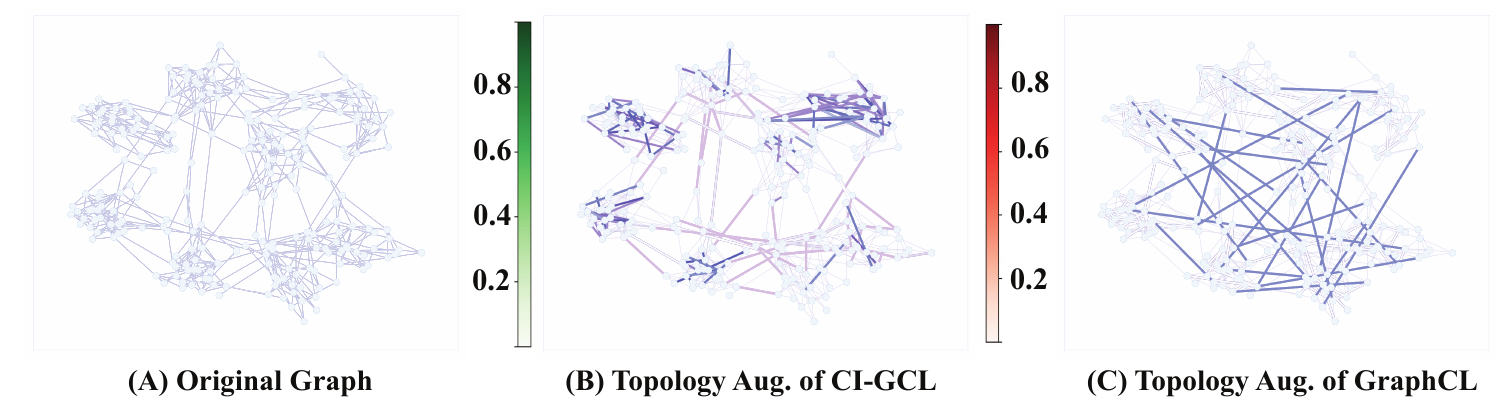}
    \caption{An case study of the community-invariant augmentation on a Random Partition Graph. In (B-C),  green lines represent edge dropping operation and red lines represent edge adding operation.}
    \label{fig:appendix_community}
\end{figure}

\subsection{SPATIAL BEHAVIOR OF SPECTRAL AUGMENTATION SPAN}
The Case Study on Random Partition Graph. 
To intuitively demonstrate the spatial change caused by spectral augmentation, we present a case study on a random geometric graph in Figure \ref{fig:appendix_community}. 
Specifically, Figure \ref{fig:appendix_community}(A) depicts the original graph, while Figure \ref{fig:appendix_community}(B) illustrates the perturbation probability with the community-invariant constraint. 
Figure \ref{fig:appendix_community}(C) demonstrates the perturbation edges of uniform augmentations. 
We observe that with our community-invariant constraint, the cluster structure is preserved after augmentation, which assigns a higher probability to removing edges between clusters and adding edges within clusters.
On the other hand, 
random topological augmentation uniformly removes and adds edges, causing the cluster effect to become blurred.

\section{Algorithms}

\setcounter{figure}{0}
\setcounter{table}{0}
\setcounter{equation}{0}
\renewcommand\theequation{E.\arabic{equation}}
\renewcommand\thefigure{E\arabic{figure}}
\renewcommand\thetable{E\arabic{table}}

Algorithm~\ref{algo_disjdecomp} illustrates the detailed steps of developing CI-ACL. 

\IncMargin{0.8em}
\begin{algorithm}[htp]
\caption{Learnable Graph Contrastive Learning with Community-Invariant Constraint}\label{algo_disjdecomp}
\SetKwData{Left}{left}\SetKwData{This}{this}\SetKwData{Up}{up}\SetKwData{Hyperparameter}{hyperparameter}
\SetKwFunction{Union}{Union}\SetKwFunction{FindCompress}{FindCompress}
\SetKwInOut{Input}{Input}\SetKwInOut{Output}{Output}
\SetKwInOut{Hyperparameters}{Hyperparameters}

\Indm
\Input{Datasets $\{G: G\in \mathcal{G}\}$, initial encoder $f_{\theta}(\cdot)$, readout function $r_{\phi}(\cdot)$, iteration number T.}
\Hyperparameters{Topological constraint coefficient $\boldsymbol{\alpha}$, feature constraint coefficient $\boldsymbol{\beta}$, number of eigenvector $K$.}
\Output{Optimized encoder $f_{\theta}(\cdot)$.}
\BlankLine
\Indp

\For{t $\leftarrow$ 1 to T}{
    \For{Sampled a minibatch of graphs $\{G_n=(\mathbf{A}_n, \mathbf{X}_n): n=1,\dots,N\}$}{
        \For{$n \leftarrow 1$ to $N$}{
            \tcp{\textcolor{blue}{Spectral Decomposition.}}
            $\mathbf{U}$ $\leftarrow$ spectral decomposition($\mathbf{A}_n$)
            
            $\mathbf{F}$ $\leftarrow$ SVD($\mathbf{X}_n$)
            
            \tcp{\textcolor{blue}{Graph Augmentation.}} 
            \For{any one edge $A_{ij}$, the $i$-th node,  and feature $X_{il}$ in $G_n$}{
            $\mathbf{\Delta}^{\text{EP}}_{ij}$ = Gumble-Softmax(MLPs($\text{Concat}\left(\mathbf{U}_{i\cdot}, \mathbf{U}_{j\cdot}\right)$))
            
            $\mathbf{\Delta}^{\text{ND}}_{ij}$ = Gumble-Softmax(MLPs($\mathbf{U}_{i\cdot}$))
            
            $\mathbf{\Delta}^{\text{FM}}_{il}$ = Gumble-Softmax(MLPs($\left[\mathbf{F}_{i\cdot}, \mathbf{F}_{l\cdot}\right]$))
            }
 \tcp{\textcolor{blue}{Sampling Two Augmented Views.}}
$t_{1}^{\text{TA}}(G_{n}) = \mathbf{A} + \mathbf{C} \circ \mathbf{\Delta^{\text{}}},\,\,$ \tcp{\textcolor{blue}{Topology Augmentation.}}

$    t_{2}^{\text{FA}}(G_{n}) = \mathbf{X} + (-\mathbf{X}) \circ \mathbf{\Delta}^{\text{FM}}$ \tcp{\textcolor{blue}{Feature Augmentation.}}

\tcp{\textcolor{blue}{GNNs Encoding Layers.}}

            Encode Topology View:  $\mathbf{z}_{n}^{(1)}=r_{\phi}(f_{\theta}(t_{1}^{\text{TA}}(G_n))$
    
            Encode Feature View: $\mathbf{z}_{n}^{(2)}=r_{\phi}(f_{\theta}(t_{2}^{\text{FA}}(G_n)))$
    
            \tcp{\textcolor{blue}{Compute Community-Invariant Constraint.}}
            Compute $\mathcal{L}_{\text{TA}}(\mathbf{\Delta}^{\text{TA}})$ to measure the spectral change as Eqs. (\ref{over-2},\ref{over-3})
            
            Compute $\mathcal{L}_{\text{FA}}(\mathbf{\Delta}^{\text{FA}})$ to measure the spectral change as Eq. (\ref{over-4})
        }
        \tcp{\textcolor{blue}{Graph Contrastive loss for the minibatch}}
        Compute $\mathcal{L}_{\text{GCL}}=-\frac{1}{N}\sum_{n=1}^{N} \left( 
        \log \frac{\exp (\text{sim}(\mathbf{z}_{n}^{(1)},\mathbf{z}_{n}^{(2)}/\tau_{2} ))}{\sum_{n'=1,n'\neq n}^{N}\exp (\text{sim}(\mathbf{z}_{n}^{(1)},\mathbf{z}_{n'}^{(2)})/\tau_{2})}
        \right)$ 

        \tcp{\textcolor{blue}{Jointly Optimize the Overall Objective Function}}
        Jointly optimize by minimizing $\mathcal{L}_{\text{overall}} = \mathcal{L}_{\text{GCL}}-\boldsymbol{\alpha} \mathcal{L}_{\text{TA}}(\mathbf{\Delta}^{\text{TA}})-\boldsymbol{\beta} \mathcal{L}_{\text{FA}}(\mathbf{\Delta}^{\text{FA}})$.
    }
}

\end{algorithm}\DecMargin{0.8em}

\newpage

\section{Proofs}

\setcounter{figure}{0}
\setcounter{table}{0}
\setcounter{equation}{0}
\setcounter{definition}{0}
\setcounter{lemma}{0}
\setcounter{theorem}{0}
\renewcommand\theequation{F.\arabic{equation}}
\renewcommand\thefigure{F\arabic{figure}}
\renewcommand\thetable{F\arabic{table}}
\renewcommand\thedefinition{F\arabic{definition}}
\renewcommand\thelemma{F\arabic{lemma}}
\renewcommand\thetheorem{F\arabic{theorem}}

\subsection{The Proof of Theorem \ref{the:abslute} in the Draft}

\begin{definition}
    \textit{(Eigenvalue Perturbation)} Assume matrix $\mathbf{A}'$,  the altered portion is represented by $\Delta\mathbf{A}=\mathbf{A}'-\mathbf{A}$, and the changed degree is denoted as $\Delta\mathbf{D}$. According to matrix perturbation theory \cite{hogben2013handbook}, the change in amplitude for the $y$-th eigenvalue can be represented as:
    \begin{equation}
        \label{eq:eigenvalue_perturbation}
        \Delta \lambda_{y}=\lambda_{y}^{\prime}-\lambda_{y}=\mathbf{u}_{y}^{\top} \Delta \mathbf{A} \mathbf{u}_{y}-\lambda_{y} \mathbf{u}_{y}^{\top} \Delta \mathbf{D} \mathbf{u}_{y}+\mathcal{O}(\|\Delta \mathbf{A}\|).
    \end{equation}
    
\end{definition}

\begin{lemma}
    If we only flip one edge $(i,j)$ on adjacency matrix $\mathbf{A}$, the change of $y$-th eigenvalue can be write as 
    \begin{equation}
        \Delta \lambda_{y}=\Delta w_{i j}\left(2 u_{y i} \cdot u_{y j}-\lambda_{y}\left(u_{y i}^{2}+u_{y j}^{2}\right)\right),  
        \label{eq:eigen_value_change_with_one_edge}
    \end{equation}
    where $u_{y i}$ is the $i$-th entry of $y$-th eigenvector $\mathbf{u}_{y}$, and $\Delta w_{i j}=   \left(1-2 A_{i j}\right)$ indicates the edge flip, i.e $\pm 1$. 
    \label{lemma:edgenvalue_pert_one_edge}
\end{lemma}

\begin{proof}
    Let $\Delta \mathbf{A}$ be a matrix with only 2 non-zero elements, namely $\Delta A_{i j}=\Delta A_{j i}=1-2 A_{i j}$ corresponding to a single edge flip $(i, j)$, and $\Delta \mathbf{D}$ the respective change in the degree matrix, i.e. $\mathbf{A}' = \mathbf{A} + \Delta \mathbf{A}$ and $\mathbf{D}' = \mathbf{D} + \Delta \mathbf{D}$. 

    Denote with $\mathbf{e}_i$ the vector of all zeros and a single one at position $i$. Then, we have $\Delta\mathbf{A}=\Delta w_{ij}(\mathbf{e}_i\mathbf{e}_j^{\top}+\mathbf{e}_j\mathbf{e}_i^{\top})$ and $\Delta \mathbf{D}=\Delta w_{ij}(\mathbf{e}_i\mathbf{e}_i^{\top}+\mathbf{e}_j\mathbf{e}_j^{\top})$. 

    Based on eigenvalue perturbation formula (\ref{eq:eigenvalue_perturbation}) by removing the high-order term $\mathcal{O}(\|\Delta \mathbf{A}\|)$, we have: 
    \begin{equation}
        \Delta \lambda_{y} \approx \mathbf{u}_{y}^{\top} (\Delta \mathbf{A} -\lambda_{y} \Delta \mathbf{D}) \mathbf{u}_{y}
        \label{eq:edge_perturbation_omit_high-order}
    \end{equation}

    Substituting $\Delta \mathbf{A}$ and $\Delta \mathbf{D}$, we conclude Eq. \ref{eq:eigen_value_change_with_one_edge}. 
\end{proof}

\begin{theorem}
    The constraint on the lowest $k$ eigenvalues of the normalized Laplacian matrix $\mathbf{L}_{\text{norm}}$ ensures the preservation of the community structure of nodes. 
    \label{theorem:community_preserve_and_eigenvalue}
\end{theorem}

\begin{proof}
    Firstly, we separate  $ \Delta \mathbf{A}=\Delta_{\mathbf{A}+}-\Delta_{\mathbf{A}-} $ , where  $ \Delta_{\mathbf{A}+} $  and  $ \Delta_{\mathbf{A}-} $  indicate which edge is added and deleted, respectively. 
    To analyze the change of eigenvalues in spectral space corresponding to the perturbation of edges in spatial space, we first consider the situation that only augments one edge for both edge-dropping ($\Delta w_{i j} = -1$) and edge-adding ($\Delta w_{i j} = 1$):

    
    \textcircled{1}  In the case of edge dropping ($\Delta w_{i j} = -1$ in Lemma \ref{lemma:edgenvalue_pert_one_edge}), we have 
    \begin{align}
    \Delta \lambda_{y}  & = -2 u_{y i} \cdot u_{y j}+\lambda_{y}\left(u_{y i}^{2}+u_{y j}^{2}\right)  \\
     & =\left(u_{y i}-u_{y j}\right)^{2} +(\lambda_{y}-1)\left(u_{y i}^{2}+u_{y j}^{2}\right)
    \end{align}

    If we only drop the edge $(i,j)$ that makes a large change in the eigenvalues. We have the objective function as
    \begin{align}
        \mathop{\arg\max}\limits_{\{i,j|\Delta w_{i j} = -1 \}}
       \sum_{y=1}^{n}  |\Delta \lambda_{y}| &  = | \left(u_{y i}-u_{y j}\right)^{2} +(\lambda_{y}-1)\left(u_{y i}^{2}+u_{y j}^{2}\right)|  \\
       & \leq  \sum_{y=1}^{n}  | \left(u_{y i}-u_{y j}\right)^{2} | + \sum_{y=1}^{n} |\lambda_{y}-1|\left(u_{y i}^{2}+u_{y j}^{2}\right) \\
       &  =  \|\mathbf{U}_{ i\cdot}-\mathbf{U}_{ j\cdot}\|^{2}  + \sum_{y=1}^{n} |\lambda_{y}-1|\left(u_{y i}^{2}+u_{y j}^{2}\right) \\
      & \leq  \|\mathbf{U}_{i\cdot}-\mathbf{U}_{j\cdot }\|^{2} + \sum_{y=1}^{n} |\lambda_{y}-1|\left(u_{y 1}^{2}+u_{y 2}^{2} + \cdots + u_{y n}^{2}\right) \\
      & =   \|\mathbf{U}_{i\cdot}-\mathbf{U}_{ j\cdot}\|^{2} + \sum_{y=1}^{n} |\lambda_{y}-1|
      \label{eq:eigenvalue_change_of_edge_drop}
    \end{align}
    Notice that $\mathbf{u}_y$ is the $y$-th column of $\mathbf{U}$, so $u_{yi}=U_{iy}$ and $\mathbf{U}_{i,\cdot}=[u_{0i}, u_{1i}, \dots, u_{ni}]$. From the first item in Eq. \ref{eq:eigenvalue_change_of_edge_drop}, we prefer to select the nodes with larger distances in the eigenvector spaces, i.e. two nodes belonging to different communities (The relationship between eigenvector and community structure is proved in Theorem \ref{theorem:relation_of_spectral_community}).

    \textcircled{2}  In the case of edge adding ($\Delta w_{i j} = 1$ in Lemma \ref{lemma:edgenvalue_pert_one_edge}), we have 
    \begin{align}
        \Delta \lambda_{y}  & = 2 u_{y i} \cdot u_{y j} - \lambda_{y}\left(u_{y i}^{2}+u_{y j}^{2}\right)  \\
        & =-\left(u_{y i}-u_{y j}\right)^{2} -(\lambda_{y}-1)\left(u_{y i}^{2}+u_{y j}^{2}\right)
    \end{align}

    If we only add the edge $(i, j)$ that makes a small change in the eigenvalues. We have the objective function as
    \begin{align}
        \mathop{\arg\min}\limits_{\{i,j|\Delta w_{i j} =+1 \}}
        \sum_{y=1}^{n}  |\Delta \lambda_{y}| &  = | \left(u_{y i}-u_{y j}\right)^{2} +(\lambda_{y}-1)\left(u_{y i}^{2}+u_{y j}^{2}\right)|  \\
        & \geq  \sum_{y=1}^{n}  | \left(u_{y i}-u_{y j}\right)^{2} | - \sum_{y=1}^{n} |1-(\lambda_{y})|\left(u_{y i}^{2}+u_{y j}^{2}\right) \\
        &  \geq  \|\mathbf{U}_{i\cdot}-\mathbf{U}_{ j\cdot}\|^{2}  - \sum_{y=1}^{n} |(1-\lambda_{y})|\left(u_{y 1}^{2}+u_{y 2}^{2} + \cdots + u_{yn}^{2}\right)\\
        & = \|\mathbf{U}_{ i\cdot}-\mathbf{U}_{ j\cdot}\|^{2}  - \sum_{y=1}^{n} |(1-\lambda_{y})|
        \label{eq:eigenvalue_change_of_edge_adding}
    \end{align}
    From the first item in Eq. \ref{eq:eigenvalue_change_of_edge_adding}, we prefer to select nodes with smaller distances in the eigenvector spaces, i.e. nodes belonging to one community (Theorem \ref{theorem:relation_of_spectral_community}). 

    Previously, we have proven that the constraint on the lowest $k$ eigenvalues of $\mathbf{L}$ ensures the preservation of community structure when we only augment one edge. Next, we will demonstrate that the perturbation of more than one edge still aligns with this theory. 

    Suppose we augment $m$ edges, similar to Lemma \ref{eq:eigen_value_change_with_one_edge}, we replace $\mathbf{\Delta}\mathbf{A}=\sum\limits_{(i,j)\in \{m\text{ edges}\}}\Delta w_{ij}(\mathbf{e}_i\mathbf{e}_j^{\top}+\mathbf{e}_j\mathbf{e}_i^{\top})$, and $\mathbf{\Delta}\mathbf{D}=\sum\limits_{(i,j)\in\{m \text{ edges}\}}\Delta w_{ij}(\mathbf{e}_i\mathbf{e}_i^{\top}+\mathbf{e}_j\mathbf{e}_j^{\top})$, we Substituting $\Delta \mathbf{A}$ and $\Delta \mathbf{D}$ of Eq. (\ref{eq:edge_perturbation_omit_high-order}), we get:
    \begin{equation}
        \Delta \lambda_{y}=\sum\limits_{(i,j)\in\{m \text{ edges}\}}\Delta w_{i j}\left(2 u_{y i} \cdot u_{y j}-\lambda_{y}\left(u_{y i}^{2}+u_{y j}^{2}\right)\right),  
    \end{equation}
    By replacing $\Delta w_{i j}\left(2 u_{y i} \cdot u_{y j}-\lambda_{y}\left(u_{y i}^{2}+u_{y j}^{2}\right)\right)$ with Eqs. (\ref{eq:eigenvalue_change_of_edge_drop}, \ref{eq:eigenvalue_change_of_edge_adding}), we could easily see that the community preserving theory is satisfied. 
\end{proof}

\subsection{Proof of Theorem \ref{the:2}}
\begin{theorem}
    Given a bipartite matrix $\mathbf{B}$, we have the adjacency matrix of the bipartite graph as:
    \begin{equation}
    \mathbf{M} = 
    \begin{bmatrix}
        \mathbf{0} & \mathbf{B} \\
        \mathbf{B}^{T} & \mathbf{0} \\
    \end{bmatrix},
    \end{equation}
    The eigenvalue of matrix $\mathbf{M}$ can be represented as $\lambda = 1-\sigma$, where $\sigma$ is the singular value of $\mathbf{B}$. And the eigenvectors of matrix $\mathbf{M}$ can be represented as the concatenation of singular vectors of $\mathbf{B}$ and $\mathbf{B}^{\top}$ \cite{DBLP:conf/kdd/Dhillon01}.
    \label{theorem:relation_svd_in_bipartite}
\end{theorem}

\begin{proof}
    The Laplacian matrix and Degree matrix of $\mathbf{M}$ can be represented as $\mathbf{L}$ and $\mathbf{D}$. We can split these into:
    \begin{equation}
        \mathbf{L} = 
        \begin{bmatrix}
            \mathbf{D}_1 & -\mathbf{B} \\
            -\mathbf{B}^{T} & \mathbf{D}_2 \\
        \end{bmatrix}, \text{ and \ }
        \mathbf{D} = 
        \begin{bmatrix}
            \mathbf{D}_1 & \mathbf{0} \\
            \mathbf{0} & \mathbf{D}_2 \\
        \end{bmatrix},
    \end{equation}
    where $\mathbf{D}_1$, $\mathbf{D}_2$ are diagonal matrices such that $D_1(i,i)=\sum_{j}B_{i,j}$ and $D_2(j,j)=\sum_{i}B_{i,j}$. Thus the generalized eigen problem $\mathbf{L}\mathbf{f}=\lambda\mathbf{D}\mathbf{f}$ may be written as:
    \begin{equation}
        \begin{bmatrix}
            \mathbf{D}_1 & -\mathbf{B} \\
            -\mathbf{B}^{T} & \mathbf{D}_2 \\
        \end{bmatrix}
        \begin{bmatrix}
            \mathbf{x} \\
            \mathbf{y} \\
        \end{bmatrix} = 
        \lambda
        \begin{bmatrix}
            \mathbf{D}_1 & \mathbf{0} \\
            \mathbf{0} & \mathbf{D}_2 \\
        \end{bmatrix}
        \begin{bmatrix}
            \mathbf{x} \\
            \mathbf{y} \\
        \end{bmatrix}
    \end{equation}
    Assuming that both $\mathbf{D}_1$ and $\mathbf{D}_2$ are nonsingular, we can rewrite the above equations as
    \begin{equation}
    \begin{aligned}
        \mathbf{D}_1^{1/2}\mathbf{x} - \mathbf{D}_1^{-1/2}\mathbf{B}\mathbf{y} &= \lambda \mathbf{D}_1^{1/2}\mathbf{x} \\
        -\mathbf{D}_2^{-1/2}\mathbf{B}^{\top}\mathbf{x} + \mathbf{D}_2^{1/2}\mathbf{y} &= \lambda \mathbf{D}_2^{1/2}\mathbf{y}
    \end{aligned}
    \end{equation}
    Letting $\mathbf{u}=\mathbf{D}_1^{1/2}\mathbf{x}$ and $\mathbf{v}=\mathbf{D}_2^{1/2}\mathbf{y}$, we can get 
    \begin{equation}
    \begin{aligned}
    \mathbf{D}_1^{-1/2}\mathbf{B}\mathbf{D}_2^{-1/2}\mathbf{v} &= (1-\lambda)\mathbf{u} \\
    \mathbf{D}_2^{-1/2}\mathbf{B}^{\top}\mathbf{D}_1^{-1/2}\mathbf{u} &= (1-\lambda)\mathbf{v}
    \end{aligned}
    \end{equation}
    These are precisely the equations that define the singular value decomposition (SVD) of the normalized matrix $\mathbf{D}_1^{-1/2}\mathbf{B}\mathbf{D}_2^{-1/2}$. 
    In particular, $\mathbf{u}$ and $\mathbf{v}$ are the left and right singular vectors of $\mathbf{D}_1^{-1/2}\mathbf{B}\mathbf{D}_2^{-1/2}$ respectively, while $(1-\lambda)$ is the corresponding singular value. Thus the $k$-th smallest eigenvalue of $\mathbf{L}_\text{norm}$ equals to $k$-th largest singular value of $\mathbf{D}_1^{-1/2}\mathbf{B}\mathbf{D}_2^{-1/2}$, $\lambda_k = 1-\sigma_k$. 
    Furthermore, the corresponding eigenvector of $\mathbf{L}_{\text{norm}}=\mathbf{D}^{-1/2}\mathbf{L}\mathbf{D}^{-1/2}$ is given by
    \begin{equation}
        \mathbf{f}_k = 
        \begin{bmatrix}
            \mathbf{u}_k \\
            \mathbf{v}_k \\
        \end{bmatrix}
    \end{equation}

    There is a computational benefit on $\mathbf{B}$ that $\mathbf{B}$ is of size $n\times d$, and $\mathbf{L}$ is of size $(n+d)\times (n+d)$. 
    
\end{proof}

\subsection{Proof of Theorem \ref{the:3}}

\begin{theorem}
    The constraint on the highest $k$ singular values of the feature matrix $\mathbf{X}$ ensures the preservation of the bipartite community structure of both nodes and features.
\end{theorem}
\begin{proof}
    Consider we have the feature matrix $\mathbf{X} \in \mathbb{R}^{n \times d}$, intuitively, the common features, containing the community structure, should be preserved as core components and the features that are irrelevant to community structure should be randomly masking. As \cite{DBLP:conf/nips/NieWDH17} studied, we could apply spectral clustering on matrix $\widetilde{\mathbf{X}}$ to obtain the co-clustering of both nodes and features. 
    \begin{equation}
    \widetilde{\mathbf{X}}=\begin{bmatrix}
        0 & \mathbf{X} \\
        \mathbf{X}^{\top} & 0
    \end{bmatrix}   
    \end{equation}
    where $\widetilde{\mathbf{X}} \in \mathbb{R}^{(n+d)\times(n+d)}$. 
    When we perform the eigen decomposition on $\widetilde{\mathbf{X}}$, the first $n$ rows of obtained eigenvectors represent the node embedding, and the rest $d$ rows are representations of features. 
    Similar to the edge-dropping situation in Theorem \ref{theorem:community_preserve_and_eigenvalue}, we could preserve the community structure by maximizing the spectral change of $\widetilde{\mathbf{X}}$. 
    The objective function could be written as
    \begin{align}
     \mathop{\arg\max}\limits_{\{i \in[1,\cdots,n],j \in [n+1,\cdots, n+d]|\Delta P_{i j} = 0 \}}
       \sum_{y=1}^{n+d}  |\Delta \lambda_{y}|  &= | \left(f_{y i}-f_{y j)}\right)^{2} +(\lambda_{y}-1)\left(f_{y i}^{2}+f_{y j}^{2}\right)|,\\
    & \geq  \sum_{y=1}^{n+d}  | \left(f_{y i}-f_{y j}\right)^{2} | - \sum_{y=1}^{n+d} |1-\lambda_{y}|\left(f_{y i}^{2}+f_{y j}^{2}\right) \\
    &  \geq  \|\mathbf{F}_{i \cdot}-\mathbf{F}_{j \cdot}\|^{2}  - \sum_{y=1}^{n+d} |1-\lambda_{y}|\left(f_{y 1}^{2}+f_{y 2}^{2} + \cdots + f_{yn}^{2}\right)\\
    & = \|\mathbf{F}_{i \cdot}-\mathbf{F}_{j \cdot}\|^{2}  - \sum_{y=1}^{n} |1-\lambda_{y}|
    \end{align}
    where $\mathbf{f}_y$ and $\lambda_y$ are the $y$-th eigenvector and eigenvalue of $\widetilde{\mathbf{X}}$'s normalized Laplacian matrix. 
    Similar to edge-dropping, we prefer to mask the feature $j$ for node $i$ when they do not belong to the same community, which means that the feature is irrelevant to the sample. 
    
    According to the relationship between eigenvectors and community structures (proved in Theorem \ref{theorem:relation_of_spectral_community}), we only need to constrain the change of several lowest eigenvalues to preserve community structures. And then we prove that $\lambda_y=1-\sigma_y$ on Theorem \ref{theorem:relation_svd_in_bipartite}, where $\sigma_y$ is the singular value of matrix $\mathbf{D}_1^{-1/2}\mathbf{X}\mathbf{D}_2^{-1/2}$. Therefore, the constraint on the highest $k$ singular values of $\mathbf{D}_1^{-1/2}\mathbf{X}\mathbf{D}_2^{-1/2}$ ensures the preservation of the bipartite community structure of both nodes and features. 
    
\end{proof}

\subsection{The relation between community and spectrum}

\begin{theorem}
    The spectral decomposition can indicate a relaxed solution of graph vertex partition. Given a graph $G$ with a Laplacian matrix $\mathbf{L}$, and the spectral decomposition is indicated as $\mathbf{D}^{-\frac{1}{2}}\mathbf{L}\mathbf{D}^{-\frac{1}{2}} = \mathbf{U}\mathbf{\Lambda} \mathbf{U}^\top$, where $\mathbf{D}$ is the degree matrix and $\mathbf{U}$, $\mathbf{\Lambda}$ are eigenvectors, eigenvalues correspondingly. The second smallest eigenvalue and its corresponding eigenvector indicate a bipartition of the graph.
    \label{theorem:relation_of_spectral_community}
\end{theorem}
\begin{proof}
    Given a partition of nodes of a graph (split $V$ into two disjoint sets $S_A$ and $S_B$), let $\mathbf{x}$ be an indicator vector for the partition, $x_i=1$ if node $i$ is in $S_A$, and $-1$, otherwise. Let $d_i$ be the degree of node $i$, and $w_{ij}$ is the weight of edge $(i,j)$. 
    Based on \cite{DBLP:conf/cvpr/ShiM97}, a normalized cut could be write as 
    \begin{equation}
        Ncut(S_A,S_B) = \frac{\sum_{x_i>0, x_j<0}-w_{ij}x_ix_j}{\sum_{x_i>0}d_i} + \frac{\sum_{x_i<0, x_j>0}-w_{ij}x_ix_j}{\sum_{x_i>0}d_i}. 
    \end{equation}
    
    The optimal partition is computed by minimizing the normalized cut: $\min_{\mathbf{x}}{Ncut(\mathbf{x})}$. 
    By setting $\mathbf{y}=(1+\mathbf{x})-b(1-\mathbf{x})$, $k=\frac{\sum_{x_i>0}d_i}{\sum_i d_i}$ and $b=\frac{k}{1-k}$, we could rewrite it as 
    \begin{equation}
        \min_{\mathbf{x}}Ncut(\mathbf{x})=\min_{\mathbf{y}}\frac{\mathbf{y}^{\top}\mathbf{L}\mathbf{y}}{\mathbf{y}^{\top}\mathbf{D}\mathbf{y}}.
    \label{eq:Ncut_minimization}
    \end{equation}
    with the condition $y_i\in{1,-b}$, and $\mathbf{y}^{\top}\mathbf{D}\mathbf{1}=0$.

    According to the Rayleigh quotient~\cite{golub2013matrix}, we can minimize Eq.~\ref{eq:Ncut_minimization} by solving the generalized eigenvalue system, if $\mathbf{y}$ is relaxed to take on real values. 
    \begin{equation}
        \mathbf{L}\mathbf{y}=\lambda\mathbf{D}\mathbf{y}
        \label{eq:generalized_eigenvalue}
    \end{equation}
    By writing $\mathbf{z}=\mathbf{D}^{\frac{1}{2}}\mathbf{y}$, we could transform the Eq.~\ref{eq:generalized_eigenvalue} to a standard eigensystem:
    \begin{equation}
        \mathbf{D}^{-\frac{1}{2}}\mathbf{L}\mathbf{D}^{-\frac{1}{2}}\mathbf{z} = \lambda\mathbf{z}
        \label{eq:laplacian_eigen_system}
    \end{equation}
    Because the Laplacian matrix is positive semidefinite~\cite{pothen1990partitioning}, we can easily verify that $\mathbf{z}_0=\mathbf{D}^{\frac{1}{2}}\mathbf{1}$ is the smallest eigenvector of Eq.~\ref{eq:laplacian_eigen_system} with eigenvalue $0$. And correspondingly, $\mathbf{y}_0=\mathbf{1}$ is the smallest eigenvector with an eigenvalue of $0$ in the general eigensystem \ref{eq:generalized_eigenvalue}, but do not satisfy the condition $\mathbf{y}^{\top}\mathbf{D}\mathbf{1}=0$.
    According to the Lemma~\ref{lemma:rayleigh_quotient_fact}, we could know that the second smallest eigenvector $\mathbf{z}_1$ is the solution of Eq.~\ref{eq:Ncut_minimization}, because $\mathbf{z}_1^{\top}\mathbf{z}_0=\mathbf{y}_1^{\top}\mathbf{D}\mathbf{1}=0$ satisfy the condition in Eq.~\ref{eq:Ncut_minimization}.

    Therefore, the second smallest eigenvalue and its corresponding eigenvector indicate a bipartition of the graph.
    While the next cut must be perpendicular to others, which is the third smallest eigenvalue corresponding eigenvector $\mathbf{z}_2^{\top}{\mathbf{z}_1}=\mathbf{z}_2^{\top}\mathbf{z}_0$. Recursively, the $k$-dimensional eigenvectors can represent the community structure of a graph to some extent.
\end{proof}

\begin{lemma}
    A simple fact about the Rayleigh quotient \cite{golub2013matrix}:
    Let $\mathbf{A}$ be a real symmetric matrix. Under the constraint that $\mathbf{x}$ is orthogonal to the $j-1$ smallest eigenvectors $\mathbf{x}_1, \mathbf{x}_2, ..., \mathbf{x}_{j-1}$, the quotient $\frac{\mathbf{x}^{\top}\mathbf{A}\mathbf{x}}{\mathbf{x}^{\top}\mathbf{x}}$ is minimized by the next smallest eigenvector $\mathbf{x}_j$ and its minimum value is the corresponding eigenvalue $\lambda_j$. 
    \label{lemma:rayleigh_quotient_fact}
\end{lemma}


\end{document}